\documentclass[envcountsame, oribibl]{llncs}
\usepackage{mathtools}
\usepackage{enumitem}
\usepackage{amssymb}
\usepackage{cite}
\usepackage{hyperref,url}
  \hypersetup{  
    bookmarksnumbered
  }
  \hypersetup{  
    breaklinks=false,
    pdfborderstyle={/S/U/W 0},  
    citebordercolor=.235 .702 .443,
    urlbordercolor=.255 .412 .882,
    linkbordercolor=.804 .149 .19,
  }
  \hypersetup{ 
    pdfauthor={Jasper De Bock and Gert de Cooman},
    pdftitle={A Desirability-Based Axiomatisation for Coherent Choice Functions},
    pdfkeywords={}
  }
\usepackage[all]{hypcap} 
\usepackage{etoolbox}
\newtoggle{arxiv}
\toggletrue{arxiv}

\title{A Desirability-Based Axiomatisation for\\Coherent Choice Functions}
\author{Jasper De Bock \and Gert de Cooman}
\institute{Ghent University, ELIS, SYSTeMS\\
\email{\{jasper.debock,gert.decooman\}@ugent.be}}

\DeclarePairedDelimiter{\group}{(}{)}

\DeclarePairedDelimiter{\set}{\{}{\}}
\DeclarePairedDelimiter{\card}{\vert}{\vert}

\newcommand{\naturals}{\mathbb{N}}

\newcommand{\reals}{\mathbb{R}}

\newcommand{\states}{\mathcal{X}}

\renewcommand{\succ}{>}
\renewcommand{\preceq}{\leq}

\newcommand{\opt}[1][]{u_{#1}}
\newcommand{\altopt}[1][]{v_{#1}}
\newcommand{\altopttoo}[1][]{w_{#1}}
\newcommand{\aopt}[1][]{a_{#1}}
\newcommand{\bopt}[1][]{b_{#1}}
\newcommand{\opts}{\mathcal{L}}
\newcommand{\posopts}{\opts_{\succ0}}
\newcommand{\nonposopts}{\opts_{\preceq0}}

\newcommand{\singposopts}{\opts_{\succ0}^{\mathrm{s}}}
\newcommand{\optset}[1][]{A_{#1}}
\newcommand{\altoptset}[1][]{B_{#1}}
\newcommand{\altoptsettoo}[1][]{C_{#1}}
\newcommand{\optsets}{\mathcal{Q}}

\newcommand{\assessment}{\mathcal{A}}
\newcommand{\desirset}[1][]{D_{#1}}

\newcommand{\desirsets}{\mathbf{D}}
\newcommand{\cohdesirsets}{\overline{\desirsets}}

\newcommand{\rejectset}[1][]{K_{#1}}
\newcommand{\maxrejectset}[1][]{\hat{K}_{#1}}
\newcommand{\natexrejectset}{\EX}

\newcommand{\rejectsets}{\mathbf{K}}
\newcommand{\cohrejectsets}{\overline{\rejectsets}}

\newcommand{\maxrejectsets}{\hat{\rejectsets}}
\newcommand{\choicefun}[1][]{C_{#1}}
\newcommand{\rejectfun}[1][]{R_{#1}}

\newcommand{\cset}[3][]{\set[#1]{#2\colon#3}}

\newcommand{\then}{\Rightarrow}

\newcommand{\ifandonlyif}{\Leftrightarrow}

\newcommand{\upset}[1]{{\uparrow\!\!{#1}}}

\DeclareMathOperator{\setposi}{Posi}

\DeclareMathOperator{\SU}{Su}

\DeclareMathOperator{\RN}{Rn}

\DeclareMathOperator{\RS}{Rs}
\DeclareMathOperator{\EX}{Ex}

\begin{document}

\maketitle

\begin{abstract}
Choice functions constitute a simple, direct and very general mathematical framework for modelling choice under uncertainty. 
In particular, they are able to represent the set-valued choices that typically arise from applying decision rules to imprecise-probabilistic uncertainty models. 
We provide them with a clear interpretation in terms of attitudes towards gambling, borrowing ideas from the theory of sets of desirable gambles, and we use this interpretation to derive a set of basic axioms. 
We show that these axioms lead to a full-fledged theory of coherent choice functions, which includes a representation in terms of sets of desirable gambles, and a conservative inference method.
\end{abstract}


\section{Introduction}

When uncertainty is described by probabilities, decision making is usually done by maximising expected utility. 
Except in degenerate cases, this leads to a unique optimal decision.
If, however, the probability measure is only partially specified---for example by lower and upper bounds on the probabilities of specific events---this method no longer works. 
Essentially, the problem is that two different probability measures that are both compatible with the given bounds may lead to different optimal decisions. 
In this context, several generalisations of maximising expected utility have been proposed; see~\cite{troffaes2007} for an nice overview.

A common feature of many such generalisations is that they yield \emph{set-valued choices}: when presented with a set of options, they generally return a subset of them. 
If this turns out to be a singleton, then we have a unique optimal decision, as before. 
If, however, it contains multiple options, this means that they are incomparable and that our uncertainty model does not allow us to choose between them. 
Obtaining a single decision then requires a more informative uncertainty model, or perhaps a secondary decision criterion, as the information present in the uncertainty model does not allow us to single out an optimal option. 
Set-valued choice is also a typical feature of decision criteria based on other uncertainty models that generalise the probabilistic ones to allow for imprecision and indecision, such as lower previsions and sets of desirable gambles.

\emph{Choice functions} provide an elegant unifying mathematical framework for studying such set-valued choice. 
They map option sets to option sets: for any given set of options, they return the corresponding set-valued choice. 
Hence, when working with choice functions, it is immaterial whether there is some underlying decision criterion. 
The primitive objects of this framework are simply the set-valued choices themselves, and the choice function that represents all these choices, serves as an uncertainty model in and by itself.

A major advantage of working with choice functions is that they allow us to impose axioms on choices, aimed at characterising what it means for choices to be rational and internally consistent; see for example the seminal work by Seidenfeld et al.~\cite{seidenfeld2010}. 
Here, we undertake a similar mission, yet approach it from a different angle. 
Rather than think of choice in an intuitive manner, we provide it with a concrete interpretation in terms of attitudes towards gambling, borrowing ideas from the theory of sets of desirable gambles \cite{couso2011:desirable,walley2000,cooman2010,cooman2011b}. 
From this interpretation alone, and nothing more, we develop a theory of coherent choice that includes a full set of axioms, a representation in terms of sets of desirable gambles, and a natural extension theorem. 

\iftoggle{arxiv}{In order to facilitate the reading, proofs and intermediate results have been relegated to the Appendix.}
{Due to length constraints, proofs have been relegated to the appendix of an extended arXiv version of this contribution~\cite{extended}.}

\section{Choice Functions}

A choice function $\choicefun$ is a set-valued operator on sets of options. 
In particular, for any set of options $\optset$, the corresponding value of $\choicefun$ is a subset $\choicefun(\optset)$ of $\optset$. 
The options themselves are typically actions amongst which a subject wishes to choose. 
As is customary in decision theory, every action has a corresponding reward that depends on the state of a variable~$X$, about which the subject is typically uncertain. 
Hence, the reward is uncertain too. 
The purpose of a choice function is to represent our subject's choices between such uncertain rewards.

Let us make this more concrete. 
First of all, the variable~$X$ takes values~$x$ in some set of states~$\states$. 
The reward that corresponds to a given option is then a function $\opt$ on $\states$. 
We will assume that this reward can be expressed in terms of a real-valued linear utility scale, allowing us to identify every option with a real-valued function on $\states$.\footnote{A more general approach, which takes options to be elements of an arbitrary vector space, encompasses the horse lottery approach, and was explored by Van Camp \cite{2017vancamp:phdthesis}. Our results here can be easily extended to this more general framework.}
We take these functions to be bounded and call them \emph{gambles}. 
We use $\opts$ to denote the set of all such gambles and also let
\begin{equation*}
\posopts\coloneqq\cset{\opt\in\opts}{\opt\geq0\text{ and }\opt\neq0}
\text{ and }
\nonposopts\coloneqq\cset{\opt\in\opts}{\opt\leq0}.
\end{equation*}

 Option sets can now be identified with subsets of $\opts$, which we call \emph{gamble sets}. 
 We restrict our attention here to \emph{finite} gamble sets and will use $\optsets$ to denote the set of all such finite subsets of $\opts$, including the empty set.

\begin{definition}[Choice function]\label{def:choicefunction}
A \emph{choice function} $\choicefun$ is a map from $\optsets$ to $\optsets$ such that $\choicefun(\optset)\subseteq\optset$ for every $\optset\in\optsets$. 
\end{definition}

Gambles in $\optset$ that do not belong to $\choicefun(\optset)$ are said to be \emph{rejected}. 
This leads to an alternative representation in terms of so-called rejection functions.

\begin{definition}[Rejection function]\label{def:rejectionfunction}
The \emph{rejection function} $\rejectfun[\choicefun]$ corresponding to a choice function $\choicefun$ is a map from $\optsets$ to $\optsets$, defined by $\rejectfun[\choicefun](\optset)\coloneqq\optset\setminus\choicefun(\optset)$ for all $\optset\in\optsets$.
\end{definition}
\noindent
Since a choice function is completely determined by its rejection function, any interpretation for rejection functions automatically implies an interpretation for choice functions. 
This allows us to focus on the former.

Our interpretation for rejection functions now goes as follows. 
Consider a subject whose uncertainty about $X$ is represented by a rejection function $\rejectfun[\choicefun]$, or equivalently, by a choice function $\choicefun$. 
Then for a given gamble set $\optset\in\optsets$, the statement that a gamble $\opt\in\optset$ is rejected from $\optset$---that is, that $\opt\in\rejectfun[\choicefun](\optset)$---is taken to mean that \emph{there is at least one gamble $\altopt$ in $\optset$ that our subject strictly prefers over $\opt$}.

This interpretation is of course still meaningless, because we have not yet explained the meaning of strict preference. 
Fortunately, that problem has already been solved elsewhere: strict preference between elements of $\opts$ has an elegant interpretation in terms of desirability~\cite{walley2000,quaeghebeur2015:statement}, and it is this interpretation that we intend to borrow here. 
To allow us to do so, we first provide a brief introduction to the theory of sets of desirable gambles.

\section{Sets of Desirable Gambles}\label{sec:SDGs}

A gamble $\opt\in\opts$ is said to be \emph{desirable} if our subject strictly prefers it over the zero gamble, meaning that rather than not gamble at all, she strictly prefers to commit to the gamble where, after the true value $x$ of the uncertain variable $X$ has been determined, she will receive the (possibly negative) reward $\opt(x)$.

A \emph{set of desirable gambles} $\desirset$ is then a subset of $\opts$, whose interpretation will be that it consists of gambles $\opt\in\opts$ that our subject considers desirable. 
The set of all sets of desirable gambles is denoted by $\desirsets$.
In order for a set of desirable gambles to represent a rational subject's beliefs, it should satisfy a number of rationality, or \emph{coherence}, criteria.

\begin{definition}
A set of desirable gambles $\desirset\in\desirsets$ is called \emph{coherent} if it satisfies the following axioms \emph{\cite{couso2011:desirable,cooman2010,cooman2011b,quaeghebeur2015:statement}}:
\begin{enumerate}[label=\textup{D}$_{\arabic*}$.,ref=\textup{D}$_{\arabic*}$,topsep=2pt,leftmargin=*]
\item $0\notin\desirset$;\label{ax:desirs:nozero}
\item\label{ax:desirs:pos} $\posopts\subseteq\desirset$;
\item\label{ax:desirs:cone} if $\opt,\altopt\in\desirset$, $\lambda,\mu\geq0$ and $\lambda+\mu>0$, then $\lambda\opt+\mu\altopt\in\desirset$.
\end{enumerate}
We denote the set of all coherent sets of desirable gambles by $\cohdesirsets$.
\end{definition}
\noindent
Axioms~\ref{ax:desirs:nozero} and~\ref{ax:desirs:pos} follow immediately from the meaning of desirability: zero cannot be strictly preferred to itself, and any gamble that is never negative but sometimes positive should be strictly preferred to the zero gamble. 
Axiom~\ref{ax:desirs:cone} is implied by the assumed linearity of our utility scale.


Every coherent set of desirable gambles $\desirset\in\cohdesirsets$ induces a binary preference order $\succ_{\desirset}$---a strict vector ordering---on $\opts$, defined by $ \opt\succ_{\desirset}\altopt\ifandonlyif\opt-\altopt\in\desirset$, for all $\opt,\altopt\in\opts$.
The intuition behind this definition is that a subject strictly prefers the uncertain reward $\opt$ over $\altopt$ if she strictly prefers trading $\altopt$ for $\opt$ over not trading at all, or equivalently, if she strictly prefers the net uncertain reward $\opt-\altopt$ over the zero gamble. 
The preference order $\succ_{\desirset}$ fully characterises $\desirset$: one can easily see that $\opt\in\desirset$ if and only if $\opt\succ_{\desirset}0$. 
Hence, sets of desirable gambles are completely determined by binary strict preferences between gambles.




\section{Sets of Desirable Gamble Sets}

Let us now go back to our interpretation for choice functions, which is that a gamble $\opt$ in $\optset$ is rejected from $\optset$ if and only if there is some gamble $\altopt$ in $\optset$ that our subject strictly prefers over $\opt$.
We will from now on interpret this preference in terms of desirability: we take it to mean that $v-u$ is desirable.
In this way, we arrive at the following interpretation for a choice function $\choicefun$. 
Consider any $\optset\in\optsets$ and $\opt\in\optset$, then
\begin{equation}\label{eq:choiceintermsofdesir}
\opt\notin\choicefun(\optset)
\ifandonlyif\opt\in\rejectfun[\choicefun](\optset)
\ifandonlyif(\exists\altopt\in\optset)\,\altopt-\opt
\text{~is desirable.}
\end{equation}
In other words, if we let $\optset-\set{\opt}\coloneqq\cset{\altopt-\opt}{\altopt\in\optset}$, then according to our interpretation, the statement that $\opt$ is rejected from $\optset$ is taken to mean that $\optset-\set{\opt}$ contains at least one desirable gamble.

A crucial observation here is that this interpretation does not require our subject to specify a set of desirable gambles.
Instead, all that is needed is for her to specify those gamble sets $\optset\in\optsets$ that to her contain at least one desirable gamble. 
We call such gamble sets \emph{desirable gamble sets} and collect them in a \emph{set of desirable gamble sets} $\rejectset\subseteq\optsets$. 
As can be seen from Equation~\eqref{eq:choiceintermsofdesir}, such a set of desirable gamble sets $\rejectset$ completely determines a choice function $\choicefun$ and its rejection function $\rejectfun[\choicefun]$:
\begin{equation*}\label{eq:choiceintermsofK}
\opt\notin\choicefun(\optset)
\ifandonlyif\opt\in\rejectfun[\choicefun](\optset)
\ifandonlyif\optset-\set{\opt}\in\rejectset,
\text{ for all $\optset\in\optsets$ and $\opt\in\optset$}.
\end{equation*}
The study of choice functions can therefore be reduced to the study of sets of desirable gamble sets. 
We will from now on work directly with the latter.
We will use the collective term \emph{choice models} for choice functions, rejection functions, and sets of desirable gamble sets.

Let $\rejectsets$ denote the set of all sets of desirable gamble sets $\rejectset\subseteq\optsets$, and consider any such~$\rejectset$. 
The first question to address is when to call $\rejectset$ \emph{coherent}: which properties should we impose on a set of desirable gamble sets in order for it to reflect a rational subject's beliefs? We propose the following axiomatisation, using $(\lambda,\mu)>0$ as a shorthand notation for `$\lambda\geq0$, $\mu\geq0$ and $\lambda+\mu>0$'.
\begin{definition}[Coherence]
A set of desirable gamble sets $\rejectset\subseteq\optsets$ is called \emph{coherent} if it satisfies the following axioms:
\begin{enumerate}[label=\textup{K}$_{\arabic*}$.,ref=\textup{K}$_{\arabic*}$,leftmargin=*,topsep=2pt,start=0]
\item\label{ax:rejects:nonempty}
$\emptyset\notin\rejectset$;
\item\label{ax:rejects:removezero}
$\optset\in\rejectset\then\optset\setminus\set{0}\in\rejectset$, for all $\optset\in\optsets$;
\item\label{ax:rejects:pos} $\set{\opt}\in\rejectset$, for all $\opt\in\posopts$;
\item\label{ax:rejects:cone} if $\optset[1],\optset[2]\in\rejectset$ and if, for all $\opt\in\optset[1]$ and $\altopt\in\optset[2]$, $(\lambda_{\opt,\altopt},\mu_{\opt,\altopt})>0$, then
\begin{equation*}
\cset{\lambda_{\opt,\altopt}\opt+\mu_{\opt,\altopt}\altopt}{\opt\in\optset[1],\altopt\in\optset[2]}
\in\rejectset;
\end{equation*}
\item\label{ax:rejects:mono}$\optset[1]\in\rejectset$ and $\optset[1]\subseteq\optset[2]\then\optset[2]\in\rejectset$, for all $\optset[1],\optset[2]\in\optsets$.
\end{enumerate}
We denote the set of all coherent sets of desirable gamble sets by $\cohrejectsets$.
\end{definition}

Since a desirable gamble set is by definition a set of gambles that contains at least one desirable gamble, Axioms~\ref{ax:rejects:nonempty} and~\ref{ax:rejects:mono} are immediate.
The other three axioms follow from the principles of desirability that also lie at the basis of Axioms~\ref{ax:desirs:nozero}--\ref{ax:desirs:cone}: the zero gamble is not desirable, the elements of $\posopts$ are all desirable, and any finite positive linear combination of desirable gambles is again desirable. 
Axioms~\ref{ax:rejects:removezero} and~\ref{ax:rejects:pos} follow naturally from the first two of these principles.
The argument for Axiom~\ref{ax:rejects:cone} is more subtle; it goes as follows. 
Since $\optset[1]$ and $\optset[2]$ are two desirable gamble sets, there must be at least one desirable gamble $\opt\in\optset[1]$ and one desirable gamble $\altopt\in\optset[2]$. 
Since for these two gambles, the positive linear combination $\lambda_{\opt,\altopt}\opt+\mu_{\opt,\altopt}\altopt$ is again desirable, we know that at least one of the elements of $\cset{\lambda_{\opt,\altopt}\opt+\mu_{\opt,\altopt}\altopt}{\opt\in\optset[1],\altopt\in\optset[2]}$ is a desirable gamble. 
Hence, it must be a desirable gamble set. 

\section{The Binary Case}\label{sec:binary}

Because we interpret them in terms of desirability, one might be inclined to think that sets of desirable gamble sets are simply an alternative representation for sets of desirable gambles. 
However, this is not the case: we will see that sets of desirable gamble sets constitute a much more general uncertainty framework than sets of desirable gambles. 
What lies behind this added generality is that it need not be known which gambles are actually desirable. 
For example, within the framework of sets of desirable gamble sets, it is possible to express the belief that at least one of the gambles $\opt$ or $\altopt$ is desirable while remaining undecided about which of them actually is; in order to express this belief, it suffices to state that $\set{\opt,\altopt}\in\rejectset$. 
This is impossible within the framework of sets of desirable gambles.

Any set of desirable gamble sets $\rejectset\in\rejectsets$ determines a unique set of desirable gambles based on its binary choices only, given by
\begin{equation*}
\desirset[\rejectset]\coloneqq\cset{\opt\in\opts}{\set{\opt}\in\rejectset}.
\end{equation*}
That choice models typically represent more than just binary choice is reflected in the fact that different $\rejectset$ can have the same $\desirset[\rejectset]$.
Nevertheless, there are sets of desirable gamble sets $\rejectset\in\rejectsets$ that \emph{are} completely characterised by a set of desirable gambles, in the sense that there is a (necessarily unique) set of desirable gambles $\desirset\in\desirsets$ such that $\rejectset=\rejectset[\desirset]$, with
\begin{equation\iftoggle{arxiv}{}{*}}\label{eq:choicefromdesir}
\rejectset[\desirset]
\coloneqq\cset{\optset\in\optsets}{\optset\cap\desirset\neq\emptyset}.
\end{equation\iftoggle{arxiv}{}{*}}
It follows from the discussion at the end of Section~\ref{sec:SDGs} that such sets of desirable gamble sets are completely determined by binary preferences between gambles. 
We therefore call them, and their corresponding choice functions, \emph{binary}. 
For any such binary set of desirable gamble sets $\rejectset$, the unique set of desirable gambles $\desirset\in\desirsets$ such that $\rejectset=\rejectset[\desirset]$ is given by $\desirset[\rejectset]$.

\begin{proposition}\label{prop:binaryiff}
Consider any set of desirable gamble sets $\rejectset\in\rejectsets$. 
Then $\rejectset$ is binary if and only if $\rejectset[{\desirset[\rejectset]}]=\rejectset$. 
\end{proposition}


The coherence of a binary set of desirable gamble sets is completely determined by the coherence of its corresponding set of desirable gambles.

\begin{proposition}\label{prop:coherence:for:binary}
Consider any binary set of desirable gamble sets $\rejectset\in\rejectsets$ and let $\desirset[\rejectset]\in\desirsets$ be its corresponding set of desirable gambles. Then $\rejectset$ is coherent if and only if $\desirset[\rejectset]$~is.
\end{proposition}

\section{Representation in Terms of Sets of Desirable Gambles}\label{sec:representation}

That there are sets of desirable gamble sets that are completely determined by a set of desirable gambles is nice, but such binary choice models are typically \emph{not} what we are interested in here, because then we could just as well use sets of desirable gambles to represent choice.
It is the non-binary coherent choice models that we have in our sights here.
But it turns out that our axioms lead to a representation result that allows us to still use sets of desirable gambles, or rather, sets of them, to completely characterise \emph{any} coherent choice model.




\begin{theorem}[Representation]\label{theo:rejectsets:representation}
Every coherent set of desirable gamble sets $\rejectset\in\cohrejectsets$ is dominated by at least one binary set of desirable gamble sets: $\cohdesirsets\group{\rejectset}\coloneqq\cset{\desirset\in\cohdesirsets}{\rejectset\subseteq\rejectset[\desirset]}\neq\emptyset$.
Moreover, $\rejectset=\bigcap\cset{\rejectset[\desirset]}{\desirset\in\cohdesirsets\group{\rejectset}}$.
\end{theorem}
\noindent
This powerful representation result allows us to incorporate a number of other axiomatisations~\cite{2017vancamp:phdthesis} as special cases in a straightforward manner, because the binary models satisfy the required axioms, and these axioms are preserved under taking arbitrary non-empty intersections.

\section{Natural Extension}\label{sec:natex}

In many practical situations, a subject will typically not specify a full-fledged coherent set of desirable gamble sets, but will only provide some partial \emph{assessment} $\assessment\subseteq\optsets$, consisting of a number of gamble sets for which she is comfortable about assessing that they contain at least one desirable gamble.
We now want to extend this assessment~$\assessment$ to a coherent set of desirable gamble sets in a manner that is as conservative---or uninformative---as possible.
This is the essence of \emph{conservative inference}.

We say that a set of desirable gamble sets $\rejectset[1]$ is less informative than (or rather, at most as informative as) a set of desirable gamble sets $\rejectset[2]$, when \mbox{$\rejectset[1]\subseteq\rejectset[2]$}: a subject whose beliefs are represented by $\rejectset[2]$ has more (or rather, at least as many) desirable gamble sets---sets of gambles that definitely contain a desirable gamble---than a subject with beliefs represented by $\rejectset[1]$.
The resulting partially ordered set $(\rejectsets,\subseteq)$ is a complete lattice with intersection as infimum and union as supremum.
The following theorem, whose proof is trivial, identifies an interesting substructure.

\begin{theorem}\label{theo:conservative:inference}
Let $\set{\rejectset[i]}_{i\in I}$ be an arbitrary non-empty family of sets of desirable gamble sets, with intersection $\rejectset\coloneqq\bigcap_{i\in I}\rejectset[i]$.
If $\rejectset[i]$ is coherent for all $i\in I$, then so is $\rejectset$.
This implies that $(\cohrejectsets,\subseteq)$ is a complete meet-semilattice. 
\end{theorem}
\noindent
This result is important, as it allows us to a extend a partially specified set of desirable gamble sets to the most conservative coherent one that includes it.
This leads to the conservative inference procedure we will call natural extension.

\begin{definition}[Consistency and natural extension]
For any assessment $\assessment\subseteq\optsets$, let $\cohrejectsets(\assessment)\coloneqq\cset{\rejectset\in\cohrejectsets}{\assessment\subseteq\rejectset}$.
We call the assessment~$\assessment$ \emph{consistent} if\/ $\cohrejectsets(\assessment)\neq\emptyset$, and we then call\/ $\natexrejectset(\assessment)\coloneqq\bigcap\cohrejectsets(\assessment)$ the \emph{natural extension} of $\assessment$.
\end{definition}
\noindent
In other words: an assessment $\assessment$ is consistent if it can be extended to some coherent rejection function, and then its natural extension $\natexrejectset(\assessment)$ is the least informative such coherent rejection function.

Our final result provides a more `constructive' expression for this natural extension and a simpler criterion for consistency. 
In order to state it, we need to introduce the set $\singposopts\coloneqq\cset{\set{\opt}}{\opt\in\posopts}$ and two operators on---transformations of---$\rejectsets$. 
The first is denoted by $\RS$, and defined by
\begin{equation*}
\RS\group{\rejectset}\coloneqq\cset{\optset\in\optsets}{(\exists\altoptset\in\rejectset)\altoptset\setminus\nonposopts\subseteq\optset}
\text{ for all $\rejectset\in\rejectsets$},
\end{equation*}
so $\RS\group{\rejectset}$ contains all gamble sets $\optset$ in $\rejectset$, all versions of $\optset$ with some of their non-positive options removed, and all supersets of such sets. 
The second is denoted by $\setposi$, and defined for all $\rejectset\in\rejectsets$ by
\begin{align*}
\setposi\group{\rejectset}\coloneqq\bigg\{
\bigg\{
\sum_{k=1}^n\lambda_{k}^{\opt[1:n]}\opt[k]
\colon
\opt[1:n]\in\times_{k=1}^n\optset[k]
\bigg\}
\colon
&n\in\naturals,(\optset[1],\dots,\optset[n])\in\rejectset^n,\\[-11pt]
&\big(\forall\opt[1:n]\in\times_{k=1}^n\optset[k]\big)\,\lambda_{1:n}^{\opt[1:n]}>0
\bigg\},
\end{align*}
where we used the notations $\opt[1:n]$ and $\lambda_{1:n}^{\opt[1:n]}$ for $n$-tuples of options $\opt[k]$ and real numbers $\lambda_{k}^{\opt[1:n]}$, $k\in\set{1,\dots,n}$, so $\opt[1:n]\in\opts^{n}$ and $\lambda_{1:n}^{\opt[1:n]}\in\reals^{n}$.
We also used $\lambda_{1:n}^{\opt[1:n]}>0$ as a shorthand for `$\lambda_k^{\opt[1:n]}\geq0$ for all $k\in\set{1,\dots,n}$ and $\sum_{k=1}^n\lambda_k^{\opt[1:n]}>0$'.

\begin{theorem}[Natural extension]\label{theo:rejectsets:consistency:and:natex}
Consider any assessment $\assessment\subseteq\optsets$. 
Then $\assessment$ is consistent if and only if\/ $\emptyset\notin\assessment$ and\/ $\set{0}\notin\setposi\group{\singposopts\cup\assessment}$.
Moreover, if $\assessment$ is consistent, then $\natexrejectset(\assessment)=\RS\group{\setposi\group{\singposopts\cup\assessment}}$.
\end{theorem}

\section{Conclusion}

Our representation result shows that binary choice \emph{is} capable of representing general coherent choice functions, provided we extend its language with a `disjunction' of desirability statements---as is implicit in our interpretation---, next to the `conjunction' and `negation' that are already implicit in the language of sets of desirable gambles---see~\cite{quaeghebeur2015:statement} for a clear exposition of the latter claim.

In addition, we have found recently that by adding a convexity axiom, and working with more general vector spaces of options to allow for the incorporation of horse lotteries, our interpretation and corresponding axiomatisation allows for a representation in terms of lexicographic sets of desirable gambles \cite{2017vancamp:phdthesis}, and therefore encompasses the one by Seidenfeld et al.~\cite{seidenfeld2010} (without archimedeanity).
We will report on these findings in more detail elsewhere.

Future work will address (i) dealing with the consequences of merging our accept-reject statement framework \cite{quaeghebeur2015:statement} with the choice function approach to decision making; (ii) discussing the implications of our axiomatisation and representation for conditioning, independence, and indifference (exchangeability); and (iii) expanding our natural extension results to deal with the computational and algorithmic aspects of conservative inference with coherent choice functions.

\section*{Acknowledgements}

This work owes a large intellectual debt to Teddy Seidenfeld, who introduced us to the topic of choice functions.
His insistence that we ought to pay more attention to non-binary choice if we wanted to take imprecise probabilities seriously, is what eventually led to this work.

The discussion in Arthur Van Camp's PhD~thesis \cite{2017vancamp:phdthesis} was the direct inspiration for our work here, and we would like to thank Arthur for providing a pair of strong shoulders to stand on.

As with most of our joint work, there is no telling, after a while, which of us two had what idea, or did what, exactly. 
We have both contributed equally to this paper.
But since a paper must have a first author, we decided it should be the one who took the first significant steps: Jasper, in this case.


\newpage
\appendix
\iftoggle{arxiv}
{
\section{Proofs and intermediate results}\label{app:proofs}

In this appendix, besides the operators that were introduced in the main text, we also require two additional ones:
\begin{equation*}
\SU\colon\rejectsets\to\rejectsets
\colon\rejectset\mapsto\SU\group{\rejectset}\coloneqq\cset{\optset\in\optsets}{(\exists\altoptset\in\rejectset)\altoptset\subseteq\optset}
\end{equation*}
and
\begin{equation*}
\RN\colon\rejectsets\to\rejectsets
\colon\rejectset\mapsto\RN\group{\rejectset}\coloneqq\cset{\optset\in\optsets}{(\exists\altoptset\in\rejectset)\altoptset\setminus\nonposopts\subseteq\optset\subseteq\altoptset}.
\end{equation*}
Applying them in sequence has the same effect as applying the operator $\RS$.

\begin{lemma}\label{lem:combineoperators}
Consider any set of desirable gamble sets $\rejectset\in\rejectsets$. 
Then
\begin{equation*}
\RS\group{\rejectset}
=\RN\group{\SU\group{\rejectset}}.
\end{equation*}
\end{lemma}

\begin{proof}
Consider any $\optset\in\RS\group{\rejectset}$, which means that there is some $\altoptset\in\rejectset$ such that $\altoptset\setminus\nonposopts\subseteq\optset$. 
Then $(\optset\cup\altoptset)\setminus\nonposopts\subseteq(\altoptset\setminus\nonposopts)\cup\optset=\optset\subseteq\optset\cup\altoptset$. 
Hence, if we let $\altoptset''\coloneqq\optset\cup\altoptset$, then $\altoptset''\setminus\nonposopts\subseteq\optset\subseteq\altoptset''$. 
Since $\altoptset\in\rejectset$ and $\altoptset\subseteq\altoptset''$ implies that $\altoptset''\in\SU\group{\rejectset}$, this allows us to conclude that $\optset\in\RN\group{\SU\group{\rejectset}}$.

Conversely, consider any $\optset\in\RN\group{\SU\group{\rejectset}}$, which means that there is some $\altoptset\in\SU\group{\rejectset}$ such that $\altoptset\setminus\nonposopts\subseteq\optset\subseteq\altoptset$. 
Then since $\altoptset\in\SU\group{\rejectset}$, there is some $\altoptset'\in\rejectset$ such that $\altoptset'\subseteq\altoptset$. 
Hence, we find that $\altoptset'\setminus\nonposopts\subseteq\altoptset\setminus\nonposopts\subseteq\optset$, which, since $\altoptset'\in\rejectset$, implies that $\optset\in\RS\group{\rejectset}$.
\qed
\end{proof}

\begin{proof}[Proposition~\ref{prop:binaryiff}]
If $\rejectset[{\desirset[\rejectset]}]=\rejectset$, then $\rejectset$ is trivially binary. 
So let us assume that $\rejectset$ is binary. 
We will prove that $\rejectset[{\desirset[\rejectset]}]=\rejectset$. 
Since $\rejectset$ is binary, there is a set of desirable gambles $\desirset'\subseteq\opts$ such that $\rejectset=\rejectset[\desirset']$. 
For any $\opt\in\opts$, this implies that
\begin{equation*}
\set{\opt}\in\rejectset
\ifandonlyif
\set{\opt}\in\rejectset[\desirset']
\ifandonlyif
\set{\opt}\cap\desirset'\neq\emptyset
\ifandonlyif
\opt\in\desirset'.
\end{equation*}
Hence, we find that $\desirset[\rejectset]=\cset{\opt\in\opts}{\set{\opt}\in\rejectset}
=\cset{\opt\in\opts}{\opt\in\desirset'}=\desirset'$, which indeed implies that $\rejectset[{\desirset[\rejectset]}]=\rejectset[\desirset']=\rejectset$.
\qed
\end{proof}

\begin{lemma}\label{lem:fromCohDtoCohK}
Consider any coherent set of desirable gambles $\desirset$. 
Then $\rejectset[\desirset]$ is a coherent set of desirable gamble sets.
\end{lemma}

\begin{proof}
We need to prove that $\rejectset[\desirset]$ is coherent, or equivalently, that it satisfies Axioms~\ref{ax:rejects:nonempty}--\ref{ax:rejects:mono}.

For Axiom~\ref{ax:rejects:nonempty}, observe that Equation~\eqref{eq:choicefromdesir} trivially implies that $\emptyset\notin\rejectset[\desirset]$. 
For Axiom~\ref{ax:rejects:removezero}, observe that $\optset\cap\desirset\neq\emptyset$ implies that $(\optset\setminus\set{0})\cap\desirset\neq\emptyset$ because we know from the coherence of $\desirset$ [Axiom~\ref{ax:desirs:nozero}] that $0\notin\desirset$. 
For Axiom~\ref{ax:rejects:pos}, observe that $\set{\opt}\in\rejectset[\desirset]$ is equivalent to $\opt\in\desirset$ for all $\opt\in\opts$, and take into account the coherence of $\desirset$ [Axiom~\ref{ax:desirs:pos}].
For Axiom~\ref{ax:rejects:cone}, consider any $\optset[1],\optset[2]\in\rejectset[\desirset]$, and let $\optset\coloneqq\cset{\lambda_{\opt,\altopt}\opt+\mu_{\opt,\altopt}\altopt}{\opt\in\optset[1],\altopt\in\optset[2]}$ for any particular choice of the $(\lambda_{\opt,\altopt},\mu_{\opt,\altopt})>0$ for all $\opt\in\optset[1]$ and $\altopt\in\altopt[2]$.
Then $\optset[1]\cap\desirset\neq\emptyset$ and $\optset[2]\cap\desirset\neq\emptyset$, so we can fix any $\opt[1]\in\optset[1]\cap\desirset$ and $\opt[2]\in\optset[2]\cap\desirset$.
The coherence of $\desirset$ [Axiom~\ref{ax:desirs:cone}] then implies that $\lambda_{\opt[1],\altopt[2]}\opt[1]+\mu_{\opt[1],\altopt[2]}\altopt[2]\in\desirset$, and therefore also $\optset\cap\desirset\neq\emptyset$, whence indeed $\optset\in\rejectset[\desirset]$.
And, finally, that $\rejectset[\desirset]$ satisfies Axiom~\ref{ax:rejects:mono} is an immediate consequence of its definition~\eqref{eq:choicefromdesir}.
\qed
\end{proof}

\begin{lemma}\label{lem:from:rejection:to:desirability}
Consider any coherent set of desirable gamble sets $\rejectset$.
Then $\desirset[\rejectset]$ is a coherent set of desirable gambles, and $\rejectset[{\desirset[\rejectset]}]\subseteq\rejectset$.
\end{lemma}

\begin{proof}
We first prove that $\desirset[\rejectset]$ is coherent, or equivalently, that it satisfies Axioms~\ref{ax:desirs:nozero}--\ref{ax:desirs:cone}.
For Axiom~\ref{ax:desirs:nozero}, observe that $0\in\desirset[\rejectset]$ implies that $\set{0}\in\rejectset$. Since $\rejectset$ satisfies Axiom~\ref{ax:rejects:removezero}, this implies that $\emptyset\in\rejectset$, contradicting Axiom~\ref{ax:rejects:nonempty}.
For Axiom~\ref{ax:desirs:pos}, observe that for any $\opt\in\opts$, $\opt\in\desirset[\rejectset]$ is equivalent to $\set{\opt}\in\rejectset$, and take into account Axiom~\ref{ax:rejects:pos}.
And, finally, for Axiom~\ref{ax:desirs:cone}, observe that $\opt,\altopt\in\desirset[\rejectset]$ implies that $\set{\opt},\set{\altopt}\in\rejectset$, and that Axiom~\ref{ax:rejects:cone} then implies that $\set{\lambda\opt+\mu\altopt}\in\rejectset$, or equivalently, that $\lambda\opt+\mu\altopt\in\desirset[\rejectset]$, for any choice of $(\lambda,\mu)>0$.

For the last statement, consider any $\optset\in\rejectset[{\desirset[\rejectset]}]$, meaning that $\optset\cap\desirset[\rejectset]\neq\emptyset$.
Consider any $\opt\in\optset\cap\desirset[\rejectset]$, then on the one hand $\opt\in\desirset[\rejectset]$, so $\set{\opt}\in\rejectset$.
But since on the other hand also $\opt\in\optset$, we see that $\set{\opt}\subseteq\optset$, and therefore Axiom~\ref{ax:rejects:mono} guarantees that $\optset\in\rejectset$.
\qed
\end{proof}

\begin{proof}[Proposition~\ref{prop:coherence:for:binary}]
First, suppose that $\desirset[\rejectset]$ is coherent. 
Lemma~\ref{lem:fromCohDtoCohK} then impies that $\rejectset[{\desirset[\rejectset]}]$ is coherent.
Hence, since we know from Proposition~\ref{prop:binaryiff} that $\rejectset=\rejectset[{\desirset[\rejectset]}]$, we find that $\rejectset$ is coherent.

Next, suppose that $\rejectset$ is coherent. 
Lemma~\ref{lem:from:rejection:to:desirability} then implies that $\desirset[\rejectset]$ is coherent as well.
\qed
\end{proof}

We will call a coherent set of desirable gamble sets $\maxrejectset$ \emph{maximal}, if it is not dominated by  any other coherent set of desirable gamble sets, and we collect all maximal coherent sets of desirable gamble sets in the set $\maxrejectsets\subseteq\cohrejectsets$: for any $\maxrejectset\in\cohrejectsets$,
\begin{equation*}
\maxrejectset\in\maxrejectsets
\ifandonlyif
\group{\forall\rejectset\in\cohrejectsets}
\group{\maxrejectset\subseteq\rejectset\then\maxrejectset=\rejectset}.
\end{equation*}

\begin{theorem}\label{theo:rejectsets:maximality}
Any coherent set of desirable gamble sets $\rejectset\in\cohrejectsets$ is dominated by some maximal coherent set of desirable gamble sets: $\cset{\maxrejectset\in\maxrejectsets}{\rejectset\subseteq\maxrejectset}\neq\emptyset$.
\end{theorem}

\begin{proof}
It is clearly enough to establish that the partially ordered set $\upset{\rejectset}\coloneqq\cset{\rejectset'\in\cohrejectsets}{\rejectset\subseteq\rejectset'}$ has a maximal element, and we use Zorn's Lemma to that effect.
So consider any chain $\mathcal{K}$ in $\upset{\rejectset}$, then we must prove that $\mathcal{K}$ has an upper bound in $\upset{\rejectset}$.
Since $\rejectset[o]\coloneqq\bigcup\mathcal{K}$ is clearly an upper bound, we are done if we can prove that $\rejectset[o]$ is coherent.

For Axiom~\ref{ax:rejects:nonempty}, simply observe that since $\emptyset$ belongs to no element of $\mathcal{K}$ [since they are all coherent], it cannot belong to their union $\rejectset[o]$.

For Axiom~\ref{ax:rejects:removezero}, consider any $\optset\in\rejectset[o]$. Then there is some $\rejectset'\in\mathcal{K}$ such that $\optset\in\rejectset'$, and since $\rejectset'$ is coherent, this implies that $\optset\setminus\set{0}\in\rejectset'\subseteq\rejectset[o]$.

For Axiom~\ref{ax:rejects:pos}, consider any $\opt\succ0$, then we know that $\set{\opt}\in\rejectset$ [since $\rejectset$ is coherent], and therefore also $\set{\opt}\in\rejectset[o]$, since $\rejectset\subseteq\rejectset[o]$.

For Axiom~\ref{ax:rejects:cone}, consider any $\optset[1],\optset[2]\in\rejectset[o]$ and, for all $\opt\in\optset[1]$ and $\altopt\in\optset[2]$, choose $(\lambda_{\opt,\altopt},\mu_{\opt,\altopt})>0$. 
Since $\optset[1],\optset[2]\in\rejectset[o]$, we know that there are $\rejectset[1],\rejectset[2]\in\mathcal{K}$ such that $\optset[1]\in\rejectset[1]$ and $\optset[2]\in\rejectset[2]$.
Since $\mathcal{K}$ is a chain, we can assume without loss of generality that $\rejectset[1]\subseteq\rejectset[2]$, and therefore $\set{\optset[1],\optset[2]}\subseteq\rejectset[2]$.
Since $\rejectset[2]$ is coherent, it follows that $\cset{\lambda_{\opt,\altopt}\opt+\mu_{\opt,\altopt}\altopt}{\opt\in\optset[1],\altopt\in\optset[2]}\in\rejectset[2]\subseteq\rejectset[o]$.

And finally, for Axiom~\ref{ax:rejects:mono}, consider any $\optset[1]\in\rejectset[o]$ and any $\optset[2]\in\optsets$ such that $\optset[1]\subseteq\optset[2]$.
Then we know that there is some $\rejectset'\in\mathcal{K}$ such that $\optset[1]\in\rejectset'$.
Since $\rejectset'$ is coherent, this implies that also $\optset[2]\in\rejectset'\subseteq\rejectset[o]$.
\qed
\end{proof}

\begin{lemma}\label{lem:binaryalternative}
A coherent set of desirable gamble sets $\rejectset$ is binary if and only if
\begin{equation}\label{eq:lem:binaryalternative}
\group{\forall\optset\in\rejectset\colon\card{\optset}\geq2}
\group{\exists\opt\in\optset}
\optset\setminus\set{\opt}\in\rejectset.
\end{equation}
\end{lemma}

\begin{proof}
First assume that $\rejectset$ is binary. 
We then know from Proposition~\ref{prop:binaryiff} that $\rejectset=\rejectset[{\desirset[\rejectset]}]$, implying that $\optset\in\rejectset\ifandonlyif\optset\cap\desirset[\rejectset]\neq\emptyset$, for all $\optset\in\optsets$.
Consider any $\optset\in\rejectset$ such that $\card{\optset}\geq2$. 
Then there is some $\altopt\in\optset\cap\desirset[\rejectset]$ such that $\optset=\set{\altopt}\cup(\optset\setminus\set{\altopt})$. 
But then $\card{\optset\setminus\set{\altopt}}\geq1$, so we can consider an element $\opt\in\optset\setminus\set{\altopt}$.
Since clearly $\altopt\in(\optset\setminus\set{\opt})\cap\desirset[\rejectset]$, we see that $(\optset\setminus\set{\opt})\cap\desirset[\rejectset]\neq\emptyset$ and therefore, that $\optset\setminus\set{\opt}\in\rejectset$.

Next assume that Equation~\eqref{eq:lem:binaryalternative} holds. 
Because of Proposition~\ref{prop:binaryiff}, it suffices to show that $\rejectset[{\desirset[\rejectset]}]=\rejectset$. 
We infer from Lemma~\ref{lem:from:rejection:to:desirability} that $\desirset[\rejectset]$ is a coherent set of desirable gambles, and that $\rejectset[{\desirset[\rejectset]}]\subseteq\rejectset$.
Assume {\itshape ex absurdo} that $\rejectset[{\desirset[\rejectset]}]\subset\rejectset$, so there is some $\optset\in\rejectset$ such that $\optset\notin\rejectset[{\desirset[\rejectset]}]$, or equivalently, such that $\optset\cap\desirset[\rejectset]=\emptyset$.
But then we must have that $\card{\optset}\geq2$, because otherwise $\optset=\set{\altopt}$ with $\altopt\notin\desirset[\rejectset]$ and therefore $\optset=\set{\altopt}\notin\rejectset$, a contradiction.
But then it follows from Equation~\eqref{eq:lem:binaryalternative} that there is some $\opt[1]\in\optset$ such that $\optset[1]\coloneqq\optset\setminus\set{\opt[1]}\in\rejectset$.
Since it follows from $\optset\cap\desirset[\rejectset]=\emptyset$ that also $\optset[1]\cap\desirset[\rejectset]=\emptyset$, we see that also $\optset[1]\notin\rejectset[{\desirset[\rejectset]}]$.
We can now repeat the same argument with $\optset[1]$ instead of $\optset$ to find that it must be that $\card{\optset[1]}\geq2$, so there is some $\opt[2]\in\optset[1]$ such that $\optset[2]\coloneqq\optset[1]\setminus\set{\opt[2]}\in\rejectset$ and $\optset[2]\notin\rejectset[{\desirset[\rejectset]}]$. 
Repeating the same argument over and over again will eventually lead to a contradiction with $\card{\optset[n]}\geq2$.
Hence it must be that $\rejectset[{\desirset[\rejectset]}]=\rejectset$.
\qed
\end{proof}

\begin{lemma}\label{lem:replacing:by:dominating:options}
Consider any set of desirable gamble sets $\rejectset\in\rejectsets$ that satisfies Axioms~\ref{ax:rejects:pos} and\/~\ref{ax:rejects:cone}.
Consider any $\optset\in\rejectset$.
Then for any $\altopt\in\optset$ and any $\altopt'\in\opts$ such that $\altopt\leq\altopt'$, the gamble set $\altoptset\coloneqq\set{\altopt'}\cup\group{\optset\setminus\set{\altopt}}$ obtained by replacing $\altopt$ in $\optset$ with the dominating option $\altopt'$ still belongs to $\rejectset$: $\altoptset\in\rejectset$.
\end{lemma}

\begin{proof}
We may assume without loss of generality that $\optset\neq\emptyset$ and that $\altopt'\neq\altopt$.
Let $\altopt''\coloneqq\altopt'-\altopt$, then $\altopt''\in\posopts$, and therefore Axiom~\ref{ax:rejects:pos} implies that $\set{\altopt''}\in\rejectset$.
Applying Axiom~\ref{ax:rejects:cone} for $\optset$ and $\set{\altopt''}$ allows us to infer that $\cset{\lambda_{\opt}\opt+\mu_{\opt}\altopt''}{\opt\in\optset}\in\rejectset$ for all possible choices of $(\lambda_{\opt},\mu_{\opt})>0$.
Choosing $(\lambda_{\opt},\mu_{\opt})\coloneqq(1,0)$ for all $\opt\in\optset\setminus\set{\altopt}$ and $(\lambda_{\altopt},\mu_{\altopt})\coloneqq(1,1)$ yields in particular that $\altoptset=\set{\altopt'}\cup(\optset\setminus\set{\altopt})\in\rejectset$.
\qed
\end{proof}

\begin{lemma}\label{lem:replacing:nonpositives:by:zero}
Consider any set of desirable gamble sets $\rejectset\in\rejectsets$ that satisfies Axioms~\ref{ax:rejects:pos} and\/~\ref{ax:rejects:cone}.
Consider any $\optset\in\rejectset$ such that $\optset\cap\nonposopts\neq\emptyset$ and any $\altopt\in\optset\cap\nonposopts$, and construct the gamble set $\altoptset\coloneqq\set{0}\cup\group{\optset\setminus\set{\altopt}}$ by replacing $\altopt$ with $0$.
Then still $\altoptset\in\rejectset$.
\end{lemma}

\begin{proof}
Immediate consequence of Lemma~\ref{lem:replacing:by:dominating:options}.
\qed
\end{proof}

\begin{proposition}\label{prop:ax:rejects:RN:equivalents}
Consider any coherent set of desirable gamble sets $\rejectset\in\cohrejectsets$, then $\RN\group{\rejectset}=\rejectset$.
\end{proposition}

\begin{proof}
That $\rejectset\subseteq\RN\group{\rejectset}$ is an immediate consequence of the definition of the $\RN$ operator. 
To prove that $\RN\group{\rejectset}\subseteq\rejectset$, consider any $\optset\in\RN\group{\rejectset}$, which means that there is some $\altoptset\in\rejectset$ such that $\altoptset\setminus\nonposopts\subseteq\optset\subseteq\altoptset$. 
We need to prove that $\optset\in\rejectset$. 
Since $\rejectset$ satisfies Axiom~\ref{ax:rejects:mono}, it suffices to prove that $\altoptset\setminus\nonposopts\in\rejectset$.

If $\altoptset\cap\nonposopts=\emptyset$, then $\altoptset\setminus\nonposopts=\altoptset\in\rejectset$. 
Therefore, without loss of generality, we may assume that $\altoptset\cap\nonposopts\neq\emptyset$. 
For any $\opt\in\altoptset\cap\nonposopts$, Lemma~\ref{lem:replacing:nonpositives:by:zero} implies that we may replace $\opt$ by $0$ and still be guaranteed that the resulting set belongs to $\rejectset$. 
Hence, we can replace all elements of $\altoptset\cap\nonposopts$ with $0$ and still be guaranteed that the result $\altoptset'\coloneqq\set{0}\cup(\altoptset\setminus\nonposopts)$ belongs to $\rejectset$.
Applying Axiom~\ref{ax:rejects:removezero} now guarantees that, indeed, $\altoptset\setminus\nonposopts=\altoptset'\setminus\set{0}\in\rejectset$.
\qed
\end{proof}

\begin{proposition}\label{prop:removal:of:nonpositives}
Consider any set of desirable gamble sets $\rejectset\in\rejectsets$.
Then\/ $\RN\group{\rejectset}$ satisfies Axiom~\ref{ax:rejects:removezero}.
Moreover, if $\rejectset$ satisfies Axioms~\ref{ax:rejects:nonempty}, \ref{ax:rejects:pos}, \ref{ax:rejects:cone} and\/~\ref{ax:rejects:mono} and does not contain $\set{0}$, then so does\/ $\RN\group{\rejectset}$.
\end{proposition}

\begin{proof}
The proof of the first statement is trivial.
For the second statement, assume that $\rejectset$ does not contain $\set{0}$, and satisfies Axioms~\ref{ax:rejects:nonempty}, \ref{ax:rejects:pos}, \ref{ax:rejects:cone} and~\ref{ax:rejects:mono}.

To prove that $\RN\group{\rejectset}$ satisfies Axiom~\ref{ax:rejects:nonempty} and does not contain $\set{0}$, assume \emph{ex absurdo} that $\emptyset\in\RN\group{\rejectset}$ or $\set{0}\in\RN\group{\rejectset}$. 
We then find that there is some $\altoptset\in\rejectset$ such that either $\altoptset\setminus\nonposopts\subseteq\emptyset\subseteq\altoptset$ or $\altoptset\setminus\nonposopts\subseteq\{0\}\subseteq\altoptset$. 
In both cases, it follows that $\altoptset\subseteq\nonposopts$. 
If $\altoptset=\emptyset$, this contradicts our assumption that $\rejectset$ satisfies Axiom~\ref{ax:rejects:nonempty}. 
If $\altoptset\neq\emptyset$, it follows from Lemma~\ref{lem:replacing:nonpositives:by:zero} that we can replace every $\opt\in\altoptset$ by $0$ and still be guaranteed that the resulting gamble set $\set{0}$ belongs to $\rejectset$, again contradicting our assumptions.

To prove that $\RN\group{\rejectset}$ satisfies Axiom~\ref{ax:rejects:pos}, simply observe that the operator $\RN$ never removes gamble sets from a set of desirable gamble sets, so the gamble sets $\set{\opt}$, $\opt\in\posopts$, which belong to $\rejectset$ by Axiom~\ref{ax:rejects:pos}, will also belong to the larger $\RN\group{\rejectset}$.

To prove that $\RN\group{\rejectset}$ satisfies Axiom~\ref{ax:rejects:cone}, consider any $\optset[1],\optset[2]\in\RN\group{\rejectset}$, meaning that there are $\altoptset[1],\altoptset[2]\in\rejectset$ such that $\altoptset[1]\setminus\nonposopts\subseteq\optset[1]\subseteq\altoptset[1]$ and $\altoptset[2]\setminus\nonposopts\subseteq\optset[2]\subseteq\altoptset[2]$.
For any $\opt\in\optset[1]$ and $\altopt\in\optset[2]$, we choose $(\lambda_{\opt,\altopt},\mu_{\opt,\altopt})>0$, and let 
\begin{equation*}
\optset
\coloneqq\cset{\lambda_{\opt,\altopt}\opt+\mu_{\opt,\altopt}\altopt}{\opt\in\optset[1],\altopt\in\optset[2]}.
\end{equation*}
Then we have to prove that $\optset\in\RN\group{\rejectset}$.
Since $\rejectset$ satisfies Axiom~\ref{ax:rejects:cone}, we infer from $\altoptset[1],\altoptset[2]\in\rejectset$ that
\begin{align*}
\altoptsettoo
\coloneqq&\cset{\lambda_{\opt,\altopt}\opt+\mu_{\opt,\altopt}\altopt}
{\opt\in\optset[1],\altopt\in\optset[2]}\\
&\qquad\cup\cset{1\opt+0\altopt}{\opt\in\altoptset[1]\setminus\optset[1],\altopt\in\altoptset[2]}
\cup\cset{0\opt+1\altopt}{\opt\in\optset[1],\altopt\in\altoptset[2]\setminus\optset[2]}\\
=&\,\optset\cup
\cset{\opt}{\opt\in\altoptset[1]\setminus\optset[1],\altopt\in\altoptset[2]}
\cup\cset{\altopt}{\opt\in\optset[1],\altopt\in\altoptset[2]\setminus\optset[2]}
\end{align*}
belongs to $\rejectset$ as well. 
Furthermore, since $\altoptset[1]\setminus\nonposopts\subseteq\optset[1]$ and $\altoptset[2]\setminus\nonposopts\subseteq\optset[2]$ imply that $\altoptset[1]\setminus\optset[1]\subseteq\nonposopts$ and $\altoptset[2]\setminus\optset[2]\subseteq\nonposopts$, we see that
\begin{equation*}
\cset{\opt}{\opt\in\altoptset[1]\setminus\optset[1],\altopt\in\altoptset[2]}
\cup\cset{\altopt}{\opt\in\optset[1],\altopt\in\altoptset[2]\setminus\optset[2]}
\subseteq\nonposopts.
\end{equation*}
Hence, $\altoptsettoo\setminus\nonposopts\subseteq\optset\subseteq\altoptsettoo$. 
Since $\altoptsettoo\in\rejectset$, this implies that, indeed, $\optset\in\RN\group{\rejectset}$.

Finally, to prove that $\RN\group{\rejectset}$ satisfies Axiom~\ref{ax:rejects:mono}, consider any $\optset[1]\in\RN\group{\rejectset}$ and any $\optset[2]\in\optsets$ such that $\optset[1]\subseteq\optset[2]$.
We need to prove that $\optset[2]\in\rejectset$.
That $\optset[1]\in\RN\group{\rejectset} $ implies that there is some $\altoptset[1]\in\rejectset$ such that $\altoptset[1]\setminus\nonposopts\subseteq\optset[1]\subseteq\altoptset[1]$.
Let $\altoptset[2]\coloneqq\altoptset[1]\cup\group{\optset[2]\setminus\optset[1]}$, then $\altoptset[1]\subseteq\altoptset[2]$ and therefore also $\altoptset[2]\in\rejectset$, because $\rejectset$ satisfies Axiom~\ref{ax:rejects:mono}.
We now infer from $\altoptset[1]\setminus\nonposopts\subseteq\optset[1]\subseteq\altoptset[1]$ that
\begin{equation*}
\altoptset[2]\setminus\nonposopts
\subseteq\group{\altoptset[1]\setminus\nonposopts}\cup\group{\optset[2]\setminus\optset[1]}
\subseteq\optset[1]\cup\group{\optset[2]\setminus\optset[1]}
\subseteq\altoptset[1]\cup\group{\optset[2]\setminus\optset[1]}.
\end{equation*}
Since $\optset[1]\cup\group{\optset[2]\setminus\optset[1]}=\optset[2]$, this allows us to conclude that $\altoptset[2]\setminus\nonposopts\subseteq\optset[2]\subseteq\altoptset[2]$, and therefore, since $\altoptset[2]\in\rejectset$, that, indeed, $\optset[2]\in\RN\group{\rejectset}$.
\qed
\end{proof}

\begin{proposition}\label{prop:Kstarstar}
Consider a coherent set of desirable gamble sets $\rejectset\in\cohrejectsets$ and any $\optset[o]\in\rejectset$ such that $\card{\optset[o]}\geq2$ and $\optset[o]\setminus\set{\opt}\notin\rejectset$ for all $\opt\in\optset[o]$. 
Choose any $\opt[o]\in\optset[o]$ and let
\begin{equation}\label{eq:prop:Kstarstar}
\rejectset^{**}
\coloneqq
\cset[\Big]{\cset[\big]{\lambda_{\altopt}\altopt+\mu_{\altopt}\opt[o]}{\altopt\in\altoptset}}
{\altoptset\in\rejectset,(\forall\altopt\in\altoptset)(\lambda_{\altopt},\mu_{\altopt})>0}.
\end{equation}
Then $\rejectset^{*}\coloneqq\RN\group{\rejectset^{**}}$ is a coherent set of desirable gamble sets that is a superset of $\rejectset$ and contains $\set{\opt[o]}$, and furthermore $\set{\opt[o]}\notin\rejectset$ and $\opt[o]\not\leq0$.
\end{proposition}

\begin{proof}
To prove that $\set{\opt[o]}\notin\rejectset$, assume {\itshape ex absurdo} that $\set{\opt[o]}\in\rejectset$.
Since $\card{\optset[o]\setminus\set{\opt[o]}}\geq1$, we can pick any element $\altopt\in\optset[o]\setminus\set{\opt[o]}$, and then $\set{\opt[o]}\subseteq\optset[o]\setminus\set{\altopt}$ and therefore $\optset[o]\setminus\set{\altopt}\in\rejectset$ by Axiom~\ref{ax:rejects:mono}, contradicting the assumptions.
To prove that $\opt[o]\not\preceq0$, assume {\itshape ex absurdo} that $\opt[o]\in\nonposopts$, then we infer that also $\optset[o]\setminus\set{\opt[o]}\in\rejectset$ [use Proposition~\ref{prop:ax:rejects:RN:equivalents} and the coherence of $\rejectset$], contradicting the assumptions. 
To prove that $\set{\opt[o]}\in\rejectset^{*}$, it suffices to notice that $\set{\opt[o]}=\cset{0\altopt+1\opt[o]}{\altopt\in\optset[o]}\in\rejectset^{**}$, whence also $\set{\opt[o]}\in \rejectset^{*}$. 
Similarly, since $\rejectset^{**}$ is clearly a superset of $\rejectset$, the same is true for $\rejectset^*$.

It only remains to prove, therefore, that $\rejectset^*$ is coherent. 
To this end, we intend to show that the set of desirable gamble sets $\rejectset^{**}$ satisfies Axioms~\ref{ax:rejects:nonempty}, \ref{ax:rejects:pos}, \ref{ax:rejects:cone} and~\ref{ax:rejects:mono} and that $\set{0}\notin\rejectset^{**}$. 
The coherence of $\rejectset^*$ will then be an immediate consequence of Proposition~\ref{prop:removal:of:nonpositives}.

For Axiom~\ref{ax:rejects:nonempty}, notice that $\emptyset\notin\rejectset$ because $\rejectset$ satisfies Axiom~\ref{ax:rejects:nonempty}. 
It therefore follows from Equation~\eqref{eq:prop:Kstarstar} that, indeed, $\emptyset\notin\rejectset^{**}$.

For Axiom~\ref{ax:rejects:pos}, consider any $\opt\in\posopts$.
Then $\set{\opt}\in\rejectset$ because $\rejectset$ satisfies Axiom~\ref{ax:rejects:pos}. 
Since $\rejectset^{**}$ is a superset of $\rejectset$, we see that, indeed, also $\set{\opt}\in\rejectset^{**}$.

For Axiom~\ref{ax:rejects:mono}, consider any $\optset[1]\in\rejectset^{**}$ and any $\optset[2]\in\optsets$ such that $\optset[1]\subseteq\optset[2]$, then we must prove that also $\optset[2]\in\rejectset^{**}$.
Since $\optset[1]\in\rejectset^{**}$, we know that there is some $\altoptset[1]\in\rejectset$ and, for all $\altopt\in\altoptset[1]$, some choice of $(\lambda_{\altopt},\mu_{\altopt})>0$,				 such that
\begin{equation*}
\optset[1]
=\cset{\lambda_{\altopt}\altopt+\mu_{\altopt}\opt[o]}{\altopt\in\altoptset[1]}.
\end{equation*}
For every $\opt\in\optset[2]\setminus\optset[1]$, we now choose some real $\alpha_{\opt}>0$ such that $\opt-\alpha_{\opt}\opt[o]\notin\altoptset[1]$ and such that, for all $\opt,\opt'\in\optset[2]\setminus\optset[1]$, $\opt-\alpha_{\opt}\opt[o]\neq\opt'-\alpha_{\opt'}\opt[o]$. 
Since $\opt[o]\neq0$ and $\optset[1]$, $\optset[2]$ and $\altoptset[1]$ are finite, this is clearly always possible.
Let 
\begin{equation*}
\altoptset[2]
\coloneqq\altoptset[1]\cup
\cset{\opt-\alpha_{\opt}\opt[o]}{\opt\in\optset[2]\setminus\optset[1]}
\end{equation*}
and, for each $\altopt\in\altoptset[2]\setminus\altoptset[1]$, let $\opt[\altopt]$ be the unique element of $\optset[2]\setminus\optset[1]$ for which $v=\opt[\altopt]-\alpha_{\opt[\altopt]}\opt[o]$, and let $(\lambda_{\altopt},\mu_{\altopt})\coloneqq(1,\alpha_{\opt[\altopt]})>0$.
We then see that
\begin{align*}
\optset[2]
&=\optset[1]\cup(\optset[2]\setminus\optset[1])\\
&=\cset{\lambda_{\altopt}\altopt+\mu_{\altopt}\opt[o]}{\altopt\in\altoptset[1]}
\cup\cset{\opt-\alpha_{\opt}\opt[o]+\alpha_{\opt}\opt[o]}{\opt\in\optset[2]\setminus\optset[1]}\\ 
&=\cset{\lambda_{\altopt}\altopt+\mu_{\altopt}\opt[o]}{\altopt\in\altoptset[1]}
\cup\cset{\altopt+\alpha_{\opt[\altopt]}\opt[o]}{\altopt\in\altoptset[2]\setminus\altoptset[1]}\\
&=\cset{\lambda_{\altopt}\altopt+\mu_{\altopt}\opt[o]}{\altopt\in\altoptset[2]}. 
\end{align*}
Furthermore, since $\altoptset[1]\in\rejectset$ and $\altoptset[1]\subseteq\altoptset[2]$, it follows from the coherence of $\rejectset$ and Axiom~\ref{ax:rejects:mono} that $\altoptset[2]\in\rejectset$. 
Hence, indeed, $\optset[2]\in\rejectset^{**}$.

For Axiom~\ref{ax:rejects:cone}, consider any $\optset[1],\optset[2]\in\rejectset^{**}$ and, for all $\opt[1]\in\optset[1]$ and $\opt[2]\in\optset[2]$, any choice of $(\alpha_{\opt[1],\opt[2]},\beta_{\opt[1],\opt[2]})>0$. 
Then we must prove that
\begin{equation*}
\altoptsettoo
\coloneqq\cset{\alpha_{\opt[1],\opt[2]}\opt[1]+\beta_{\opt[1],\opt[2]}\opt[2]}
{\opt[1]\in\optset[1],\opt[2]\in\optset[2]}
\in\rejectset^{**}.
\end{equation*}
Since $\optset[1],\optset[2]\in\rejectset^{**}$, there are $\altoptset[1],\altoptset[2]\in\rejectset$ and, for all $\altopt[1]\in\altoptset[1]$ and $\altopt[2]\in\altoptset[2]$, some choices of $(\lambda_{1,\altopt[1]},\mu_{1,\altopt[1]})>0$ and $(\lambda_{2,\altopt[2]},\mu_{2,\altopt[2]})>0$, such that
\begin{equation*}
\optset[1]
=\cset{\lambda_{1,\altopt[1]}\altopt[1]+\mu_{1,\altopt[1]}\opt[o]}
{\altopt[1]\in\altoptset[1]}
\text{ and }
\optset[2]
=\cset{\lambda_{2,\altopt[2]}\altopt[2]+\mu_{2,\altopt[2]}\opt[o]}
{\altopt[2]\in\altoptset[2]}.
\end{equation*}
Now fix any $\altopt[1]\in\altoptset[1]$ and $\altopt[2]\in\altoptset[2]$, and let
$(\alpha'_{\altopt[1],\altopt[2]},\beta'_{\altopt[1],\altopt[2]})\coloneqq(\alpha_{\opt[1],\opt[2]},\beta_{\opt[1],\opt[2]})>0$, with
$\opt[1]\coloneqq\lambda_{1,\altopt[1]}\altopt[1]+\mu_{1,\altopt[1]}\opt[o]$
and
$\opt[2]\coloneqq\lambda_{2,\altopt[2]}\altopt[2]+\mu_{2,\altopt[2]}\opt[o]$. 
Then
\begin{align*}
\altoptsettoo
&=\cset{\alpha'_{\altopt[1],\altopt[2]}\group{\lambda_{1,\altopt[1]}\altopt[1]+\mu_{1,\altopt[1]}\opt[o]}+\beta'_{\altopt[1],\altopt[2]}\group{\lambda_{2,\altopt[2]}\altopt[2]+\mu_{2,\altopt[2]}\opt[o]}}
{\altopt[1]\in\altoptset[1],\altopt[2]\in\altoptset[2]}
\end{align*}
We consider two cases. 
If $\alpha'_{\altopt[1],\altopt[2]}\lambda_{1,\altopt[1]}+\beta'_{\altopt[1],\altopt[2]}\lambda_{2,\altopt[2]}>0$, we let
\begin{align*}
(\kappa_{\altopt[1],\altopt[2]},\rho_{\altopt[1],\altopt[2]})
&\coloneqq(\alpha'_{\altopt[1],\altopt[2]}\lambda_{1,\altopt[1]},\beta'_{\altopt[1],\altopt[2]}\lambda_{2,\altopt[2]})>0,\\
(\gamma_{\altopt[1],\altopt[2]},\delta_{\altopt[1],\altopt[2]})
&\coloneqq(1,\alpha'_{\altopt[1],\altopt[2]}\mu_{1,\altopt[1]}+\beta'_{\altopt[1],\altopt[2]}\mu_{2,\altopt[2]})>0.
\end{align*}
If $\alpha'_{\altopt[1],\altopt[2]}\lambda_{1,\altopt[1]}+\beta'_{\altopt[1],\altopt[2]}\lambda_{2,\altopt[2]}=0$, we let
\begin{align*}
(\kappa_{\altopt[1],\altopt[2]},\rho_{\altopt[1],\altopt[2]})
&\coloneqq(1,1)>0,\\
(\gamma_{\altopt[1],\altopt[2]},\delta_{\altopt[1],\altopt[2]})
&\coloneqq(0,\alpha'_{\altopt[1],\altopt[2]}\mu_{1,\altopt[1]}+\beta'_{\altopt[1],\altopt[2]}\mu_{2,\altopt[2]})>0.
\end{align*}
In both cases, we find that
\begin{multline}\label{eq:uglyproof}
\gamma_{\altopt[1],\altopt[2]}(\kappa_{\altopt[1],\altopt[2]}\altopt[1]+\rho_{\altopt[1],\altopt[2]}\altopt[2])+\delta_{\altopt[1],\altopt[2]}\opt[o]\\
=
\alpha'_{\altopt[1],\altopt[2]}(
\lambda_{1,\altopt[1]}\altopt[1]+\mu_{1,\altopt[1]}\opt[o])
+\beta'_{\altopt[1],\altopt[2]}(
\lambda_{2,\altopt[2]}\altopt[2]+\mu_{2,\altopt[2]}\opt[o])\in\altoptsettoo. 
\end{multline}
Now let
\begin{equation*}
\altoptset\coloneqq\cset[\big]{\kappa_{\altopt[1],\altopt[2]}\altopt[1]+\rho_{\altopt[1],\altopt[2]}\altopt[2]}{\altopt[1]\in\altoptset[1],\altopt[2]\in\altoptset[2]}.
\end{equation*}
Then clearly, for all $\altopttoo\in\altoptset$, there are $\altopt[1]\in\altoptset[1]$ and $\altopt[2]\in\altoptset[2]$ such that $\altopttoo=\kappa_{\altopt[1],\altopt[2]}\altopt[1]+\rho_{\altopt[1],\altopt[2]}\altopt[2]$. 
However, there could be multiple such pairs. 
We choose any one such pair and denote its two elements by $\altopt[1,\altopttoo]$ and $\altopt[2,\altopttoo]$, respectively. 
Using this notation, we now define the set
\begin{equation*}
\altoptsettoo'
\coloneqq
\cset[\big]{\gamma_{\altopt[1,\altopttoo],\altopt[2,\altopttoo]}\altopttoo+\delta_{\altopt[1,\altopttoo],\altopt[2,\altopttoo]}\opt[o]}
{\altopttoo\in\altoptset}.
\end{equation*}
Since $\altoptset[1],\altoptset[2]\in\rejectset$, the coherence of $\rejectset$ [Axiom~\ref{ax:rejects:cone}] implies that $\altoptset\in\rejectset$, which in turn implies that $\altoptsettoo'\in K^{**}$. 
Also, since
\begin{align*}
\altoptsettoo'
=&
\cset[\big]{\gamma_{\altopt[1,\altopttoo],\altopt[2,\altopttoo]}\altopttoo+\delta_{\altopt[1,\altopttoo],\altopt[2,\altopttoo]}\opt[o]}
{\altopttoo\in\altoptset}\\
=&\cset[\big]{\gamma_{\altopt[1,\altopttoo],\altopt[2,\altopttoo]}\big(\kappa_{\altopt[1,\altopttoo],\altopt[2,\altopttoo]}\altopt[1,\altopttoo]+\rho_{\altopt[1,\altopttoo],\altopt[2,\altopttoo]}\altopt[2,\altopttoo]\big)+\delta_{\altopt[1,\altopttoo],\altopt[2,\altopttoo]}\opt[o]}
{\altopttoo\in\altoptset},
\end{align*}
we infer from Equation~\eqref{eq:uglyproof} that $\altoptsettoo'\subseteq\altoptsettoo$. 
Since we have already proved that $\rejectset^{**}$ satisfies Axiom~\ref{ax:rejects:mono}, this implies that, indeed, $\altoptsettoo\in\rejectset^{**}$.

It therefore now only remains to prove that $\set{0}\notin\rejectset^{**}$. 
So assume {\itshape ex absurdo} that $\set{0}\in\rejectset^{**}$, meaning that there is some $\altoptset\in\rejectset$ and, for all $\altopt\in\altoptset$, some choice of $(\lambda_{\altopt},\mu_{\altopt})>0$, such that $\cset{\lambda_{\altopt}\altopt+\mu_{\altopt}\opt[o]}{\altopt\in\altoptset}=\set{0}$. 
Hence, $\altoptset\neq\emptyset$ and $\lambda_{\altopt}\altopt+\mu_{\altopt}\opt[o]=0$ for all $\altopt\in\altoptset$.

Recall that we already know that $\opt[o]\neq0$. 
For any $\altopt\in\altoptset$, $\lambda_{\altopt}\altopt+\mu_{\altopt}\opt[o]=0$ implies that $\lambda_{\altopt}>0$, because otherwise, since $(\lambda_{\altopt},\mu_{\altopt})>0$, $\lambda_{\altopt}=0$ would imply that $\mu_{\altopt}>0$ and therefore $\opt[o]=0$, a contradiction.
Hence, for all $\altopt\in\altoptset$, $\altopt=-\delta_{\altopt}\opt[o]$ with $\delta_{\altopt}\coloneqq\frac{\mu_{\altopt}}{\lambda_{\altopt}}\geq0$. 
Now let $(\kappa_{\opt,\altopt},\rho_{\opt,\altopt})\coloneqq(1,0)$ for all $\opt\in\optset[o]\setminus\set{\opt[o]}$ and $\altopt\in\altoptset$, and let $(\kappa_{\opt[o],\altopt},\rho_{\opt[o],\altopt})\coloneqq(\delta_{\altopt},1)$ for all $\altopt\in\altoptset$. 
Then
\begin{multline*}
\cset{\kappa_{\opt,\altopt}\opt+\rho_{\opt,\altopt}\altopt}{\opt\in\optset[o],\altopt\in\altoptset}\\
\begin{aligned}
&=
\cset{\opt}{\opt\in\optset[o]\setminus\set{\opt[o]},\altopt\in\altoptset}
\cup
\cset{\delta_{\altopt}\opt[o]+\altopt}{\altopt\in\altoptset}\\
&=
\cset{\opt}{\opt\in\optset[o]\setminus\set{\opt[o]},\altopt\in\altoptset}
\cup
\cset{0}{\altopt\in\altoptset}\\
&=\set{0}\cup\group{\optset[o]\setminus\set{\opt[o]}},
\end{aligned}
\end{multline*}
where the last equality follows from $\altoptset\neq\emptyset$. 
However, since $\optset[o]\in\rejectset$ and $\altoptset\in\rejectset$, the coherence of $\rejectset$ [Axiom~\ref{ax:rejects:cone}] implies that $\cset{\kappa_{\opt,\altopt}\opt+\rho_{\opt,\altopt}\altopt}{\opt\in\optset[o],\altopt\in\altoptset}
\in\rejectset$. 
We therefore find that $\set{0}\cup\group{\optset[o]\setminus\set{\opt[o]}}\in\rejectset$. 
The coherence of $\rejectset$ now guarantees that $\optset[o]\setminus\set{\opt[o]}\in\rejectset$ [use Axiom~\ref{ax:rejects:removezero} if $\set{0}\notin\optset[o]\setminus\set{\opt[o]}$], contradicting the assumptions.
\qed
\end{proof}


\begin{proposition}\label{prop:nonbinary:is:dominated}
Any coherent non-binary set of desirable gamble sets  $\rejectset$ is \emph{strictly dominated}, meaning that there is some coherent set of desirable gamble sets $\rejectset^*$ such that $\rejectset\subset\rejectset^*$.
\end{proposition}

\begin{proof}
Consider an arbitrary coherent non-binary set of desirable gamble sets $\rejectset$. 
We infer from Lemma~\ref{lem:binaryalternative} that there is some $\optset[o]\in\rejectset$ such that $\card{\optset[o]}\geq2$ and $\optset[o]\setminus\set{\opt}\notin\rejectset$ for all $\opt\in\optset[o]$.
Consider any $\opt[o]\in\optset[o]$ and let $\rejectset^*\coloneqq\RN\group{\rejectset^{**}}$, with $\rejectset^{**}$ as in Equation~\eqref{eq:prop:Kstarstar}. 
It then follows from Proposition~\ref{prop:Kstarstar} that $\rejectset^{*}$ is a coherent set of desirable gamble sets that is a superset of $\rejectset$ and contains $\set{\opt[o]}$, and that $\set{\opt[o]}\notin\rejectset$. 
Hence, $\rejectset\subset\rejectset^*$.
\qed
\end{proof}




\begin{proof}[Theorem~\ref{theo:rejectsets:representation}]
Let $\rejectset[o]$ be a coherent set of desirable gamble sets. 
We prove that $\cohdesirsets(\rejectset[o])\coloneqq\cset{\desirset\in\cohdesirsets}{\rejectset[o]\subseteq\rejectset[\desirset]}\neq\emptyset$ and that $\rejectset[o]=\bigcap\cset{\rejectset[\desirset]}{\desirset\in\cohdesirsets(\rejectset[o])}$.

For the first statement, recall from Theorem~\ref{theo:rejectsets:maximality} that there is some maximal coherent set of desirable gamble sets $\maxrejectset\in\maxrejectsets$ that dominates $\rejectset[o]$: $\rejectset[o]\subseteq\maxrejectset$.
Assume \emph{ex absurdo} that $\maxrejectset$ is non-binary. 
It then follows from Proposition~\ref{prop:nonbinary:is:dominated} that $\maxrejectset$ is strictly dominated, contradicting its maximality. 
Hence, it must be that $\maxrejectset$ is binary. 
Proposition~\ref{prop:binaryiff} therefore implies that $\maxrejectset=\rejectset[\desirset]$, with $\smash{\desirset=\desirset[\maxrejectset]}$. 
Furthermore, because $\maxrejectset$ is coherent, Proposition~\ref{prop:coherence:for:binary} implies that $\desirset$ is coherent, whence $\desirset\in\cohdesirsets$. 
Since $\rejectset[o]\subseteq\maxrejectset=\rejectset[\desirset]$, $\cohdesirsets(\rejectset[o])\coloneqq\cset{\desirset\in\cohdesirsets}{\rejectset[o]\subseteq\rejectset[\desirset]}\neq\emptyset$.

For the second statement, it is obvious that $\rejectset[o]\subseteq\bigcap\cset{\rejectset[\desirset]}{\desirset\in\cohdesirsets(\rejectset[o])}$, so we concentrate on the proof of the converse inclusion.
Assume {\itshape ex absurdo} that $\rejectset[o]\subset\bigcap\cset{\rejectset[\desirset]}{\desirset\in\cohdesirsets(\rejectset[o])}$, so there is some gamble set $\altoptset[o]\in\optsets$ such that $\altoptset[o]\notin\rejectset[o]$ and $\altoptset[o]\in\rejectset[\desirset]$ for all $\desirset\in\cohdesirsets(\rejectset[o])$, so $\altoptset[o]\neq\emptyset$.
Then $\altoptset[o]\setminus\nonposopts\notin\rejectset[o]$ [use the coherence of $\rejectset[o]$ and Axiom~\ref{ax:rejects:mono}] and $\altoptset[o]\setminus\nonposopts\in\rejectset[\desirset]$ for all $\desirset\in\cohdesirsets(\rejectset[o])$ [use the coherence of $\rejectset[\desirset]$ (which follows from Lemma~\ref{lem:fromCohDtoCohK} and the coherence of $\desirset$) and Proposition~\ref{prop:ax:rejects:RN:equivalents}], so we may assume without loss of generality that $\altoptset[o]$ has no non-positive gambles: $\altoptset[o]\cap\nonposopts=\emptyset$.

The partially ordered set $\upset{\rejectset[o]^*}\coloneqq\cset{\rejectset\in\cohrejectsets}{\rejectset[o]\subseteq\rejectset\text{ and }\altoptset[o]\notin\rejectset}$ is non-empty because it contains $\rejectset[o]$.
An argument involving Zorn's Lemma, analogous to the one in the proof of Theorem~\ref{theo:rejectsets:maximality}, allows us to prove that this partially ordered set has maximal elements.
If we can prove that any such maximal element $\maxrejectset$ is binary, then we know from Propositions~\ref{prop:binaryiff} and~\ref{prop:coherence:for:binary} that there is some coherent set of desirable gambles $\smash{\desirset[o]=\desirset[\maxrejectset]}$ such that $\rejectset[o]\subseteq\rejectset[{\desirset[o]}]$---and therefore $\desirset[o]\in\cohdesirsets(\rejectset[o])$---and $\altoptset[o]\notin\rejectset[{\desirset[o]}]$, a contradiction.
To prove that the maximal elements of $\upset{\rejectset[o]^*}$ are binary, it suffices to prove that any non-binary element of $\upset{\rejectset[o]^*}$ is strictly dominated in that set, which is what we now set out to do.

So consider any non-binary element $\rejectset$ of $\upset{\rejectset[o]^*}$, so in particular $\rejectset\in\cohrejectsets$, $\rejectset[o]\subseteq\rejectset$ and $\altoptset[o]\notin\rejectset$.
Since $\rejectset$ is non-binary, it follows from Lemma~\ref{lem:binaryalternative} that there is some $\optset[o]\in\rejectset$ such that $\card{\optset[o]}\geq2$ and $\optset[o]\setminus\set{\opt}\notin\rejectset$ for all $\opt\in\optset[o]$.
The partially ordered set $\cset{\optset\in\rejectset}{\altoptset[o]\subseteq\optset}$ contains $\optset[o]\cup\altoptset[o]$ [because $\optset[o]\in\rejectset$ and because $\rejectset$ satisfies Axiom~\ref{ax:rejects:mono}] and therefore has some minimal (non-dominating) element $\altoptset^*$ below it, so $\altoptset^*\in\rejectset$ and $\altoptset[o]\subseteq\altoptset^*\subseteq\optset[o]\cup\altoptset[o]$.

Let us first summarise what we know about this minimal element $\altoptset^*$.
It is impossible that $\altoptset^*\subseteq\altoptset[o]$ because otherwise $\altoptset[o]=\altoptset^*\in\rejectset$, a contradiction.
Hence $\altoptset^*\setminus\altoptset[o]\neq\emptyset$, so we can fix some element $\opt[o]$ in $\altoptset^*\setminus\altoptset[o]\subseteq\optset[o]$.
Since $\altoptset[o]\subseteq\altoptset^*\setminus\set{\opt[o]}$ but $\altoptset^*\setminus\set{\opt[o]}\subset\altoptset^*$, it must be that $\altoptset^*\setminus\set{\opt[o]}\notin\rejectset$, by the definition of a minimal element.
Observe that $\altoptset^*\neq\emptyset$.

Let $\rejectset^*\coloneqq\RN\group{\rejectset^{**}}$, with $\rejectset^{**}$ as in Equation~\eqref{eq:prop:Kstarstar}. 
Since $\opt[o]\in\optset[o]$, it then follows from Proposition~\ref{prop:Kstarstar} that $\rejectset^{*}$ is a coherent set of desirable gamble sets that is a superset of $\rejectset$---and therefore also of $\rejectset[o]$---and contains $\set{\opt[o]}$, and that $\set{\opt[o]}\notin\rejectset$ and $\opt[o]\not\leq0$. 
Hence, it follows that $\rejectset\subset\rejectset^*$. 
If we can now prove that $\altoptset[o]\notin\rejectset^*$ and therefore $\rejectset^*\in\upset{\rejectset[o]^*}$, we are done, because then $\rejectset$ is indeed strictly dominated by $\rejectset^*$ in $\upset{\rejectset[o]^*}$.

Assume therefore {\itshape ex absurdo} that $\altoptset[o]\in\rejectset^*=\RN\group{\rejectset^{**}}$.
Taking into account Equation~\eqref{eq:prop:Kstarstar}, this implies that there are $\altoptsettoo\in\rejectset$ and $(\lambda_{\altopt},\mu_{\altopt})>0$ for all $\altopt\in\altoptsettoo$, such that $\cset{b_{\altopt}}{\altopt\in\altoptsettoo}\setminus\nonposopts\subseteq\altoptset[o]\subseteq\cset{b_{\altopt}}{\altopt\in\altoptsettoo}$, where, for all $\altopt\in\altoptsettoo$, $b_{\altopt}\coloneqq\lambda_{\altopt}\altopt+\mu_{\altopt}\opt[o]$.
Given our assumption that $\altoptset[o]\cap\nonposopts=\emptyset$, this also implies that $\cset{b_{\altopt}}{\altopt\in\altoptsettoo}\setminus\altoptset[o]\subseteq\nonposopts$. 
Now let $\altoptsettoo[1]\coloneqq\cset{\altopt\in\altoptsettoo}{b_{\altopt}\in\altoptset[o]}$ and $\altoptsettoo[2]\coloneqq\cset{\altopt\in\altoptsettoo}{b_{\altopt}\notin\altoptset[o]}$. 
Then $\altoptsettoo[1]\neq\emptyset$ [because $\altoptset[o]\neq\emptyset$] and $\cset{b_{\altopt}}{\altopt\in\altoptsettoo[1]}=\altoptset[o]$.
Consider now any $\altopt\in\altoptsettoo[2]$. 
Then $b_{\altopt}\notin\altoptset[o]$. 
Since $\cset{b_{\altopt}}{\altopt\in\altoptsettoo}\setminus\altoptset[o]\subseteq\nonposopts$, this implies that $b_{\altopt}=\lambda_{\altopt}\altopt+\mu_{\altopt}\opt[o]\leq0$.
Hence, we must have that $\lambda_{\altopt}>0$, because otherwise $\mu_{\altopt}\opt[o]\leq0$ with $\mu_{\altopt}>0$, and therefore also $\opt[o]\leq0$, contradicting what we inferred earlier from Proposition~\ref{prop:Kstarstar}.
So we find that
\begin{equation*}
\altopt\leq-\frac{\mu_{\altopt}}{\lambda_{\altopt}}\opt[o]
\text{ for all $\altopt\in\altoptsettoo[2]$}.
\end{equation*}
Consequently, and because $\altoptsettoo[1]\cup\altoptsettoo[2]=\altoptsettoo\in\rejectset$, we infer from Lemma~\ref{lem:replacing:by:dominating:options} that
\begin{equation*}
\altoptsettoo'
\coloneqq
\altoptsettoo[1]
\cup\cset[\Big]{-\frac{\mu_{\altopt}}{\lambda_{\altopt}}\opt[o]}{\altopt\in\altoptsettoo[2]}\in\rejectset.
\end{equation*}
Let $\altoptsettoo[3]\coloneqq\altoptsettoo'\setminus\altoptsettoo[1]$. 
Then for all $\altopt\in\altoptsettoo[3]$, there is some $\gamma_{\altopt}\geq0$ such that $\altopt=-\gamma_{\altopt}\opt[o]$.
Now let $(\alpha_{\opt[o],\altopt},\beta_{\opt[o],\altopt})\coloneqq(\mu_{\altopt},\lambda_{\altopt})$ for all $\altopt\in\altoptsettoo[1]$ and $(\alpha_{\opt[o],\altopt},\beta_{\opt[o],\altopt})\coloneqq(\gamma_{\altopt},1)$ for all $\altopt\in\altoptsettoo[3]$ and, for all $\opt\in\altoptset^*\setminus\set{\opt[o]}$ and $\altopt\in\altoptsettoo'$, let $(\alpha_{\opt,\altopt},\beta_{\opt,\altopt})\coloneqq(1,0)$. 
Then
\begin{multline*}
\cset{\alpha_{\opt,\altopt}\opt+\beta_{\opt,\altopt}\altopt}{\opt\in\altoptset^*,\altopt\in\altoptsettoo'}\\
\begin{aligned}
&=\cset{\mu_{\altopt}\opt[o]+\lambda_{\altopt}\altopt}{\altopt\in\altoptsettoo[1]}
\cup
\cset{\gamma_{\altopt}\opt[o]+\altopt}{\altopt\in\altoptsettoo[3]}
\cup
\cset{\opt}{\opt\in\altoptset^*\setminus\set{\opt[o]},\altopt\in\altoptsettoo'}\\
&=
\cset{b_{\altopt}}{\altopt\in\altoptsettoo[1]}
\cup\cset{0}{\altopt\in\altoptsettoo[3]}
\cup\cset{\opt}{\opt\in\altoptset^*\setminus\set{\opt[o]}}\\
&=\altoptset[o]\cup\cset{0}{\altopt\in\altoptsettoo[3]}\cup(\altoptset^*\setminus\set{\opt[o]})\\
&=(\altoptset^*\setminus\set{\opt[o]})\cup\cset{0}{\altopt\in\altoptsettoo[3]},
\end{aligned}
\end{multline*}
where the second equality holds because $\altoptsettoo'\in\rejectset$ and Axiom~\ref{ax:rejects:nonempty} imply that $\emptyset\neq\altoptsettoo'$, and where the fourth equality holds because $\altoptset[o]\subseteq\altoptset^*\setminus\set{\opt[o]}$. 
Since $\altoptset^*\in\rejectset$ and $\altoptsettoo'\in\rejectset$, we can now invoke Axiom~\ref{ax:rejects:cone} to find that
\begin{equation*}
\altoptset^*\setminus\set{\opt[o]}\cup\cset{0}{\altopt\in\altoptsettoo[3]}
=
\cset{\alpha_{\opt,\altopt}\opt+\beta_{\opt,\altopt}\altopt}{\opt\in\altoptset^*,\altopt\in\altoptsettoo'}\in\rejectset.
\end{equation*}
If $\altoptsettoo[3]=\emptyset$, we find that $\altoptset^*\setminus\set{\opt[o]}\in\rejectset$, a contradiction.
If $\altoptsettoo[3]\neq\emptyset$, we find that $\set{0}\cup\altoptset^*\setminus\set{\opt[o]}\in\rejectset$.
If $0\in\altoptset^*\setminus\set{\opt[o]}$, then we get that $\altoptset^*\setminus\set{\opt[o]}\in\rejectset$, a contradiction.
And if $0\notin\altoptset^*\setminus\set{\opt[o]}$, then we can still derive from Axiom~\ref{ax:rejects:removezero} that $\altoptset^*\setminus\set{\opt[o]}\in\rejectset$, again a contradiction.
\qed
\end{proof}

\begin{proposition}\label{prop:applying:posi}
For any set of desirable gamble sets $\rejectset\in\rejectsets$, $\setposi\group{\rejectset}$ satisfies Axiom~\ref{ax:rejects:cone}.
\end{proposition}

\begin{proof}
To prove that $\setposi\group{\rejectset}$ satisfies Axiom~\ref{ax:rejects:cone}, consider any $\optset,\altoptset\in\setposi\group{\rejectset}$ and, for all $\opt\in\optset$ and $\altopt\in\altoptset$, any $(\lambda_{\opt,\altopt},\mu_{\opt,\altopt})>0$. 
Then we need to prove that
\begin{equation*}
C
\coloneqq\cset{\lambda_{\opt,\altopt}\opt+\mu_{\opt,\altopt}\altopt}
{\opt\in\optset,\altopt\in\altoptset}\in\setposi\group{\rejectset}
\end{equation*}
Since $\optset,\altoptset\in\setposi\group{\rejectset}$, we know that there are $m,n\in\naturals$, $(\optset[1],\dots,\optset[m])\in\rejectset^m$ and $(\altoptset[1],\dots,\altoptset[n])\in\rejectset^n$ and, for all $\opt[1:m]\in\times_{k=1}^m\optset[k]$ and $\altopt[1:n]\in\times_{\ell=1}^n\altoptset[\ell]$, some choice of $\lambda^{\opt[1:m]}_{1:m}>0$ and $\mu^{\altopt[1:n]}_{1:n}>0$ such that
\begin{equation}\label{eq:applying:posi:first}
\optset
=\cset[\bigg]{\smashoperator[r]{\sum_{k=1}^m}
\lambda^{\opt[1:m]}_k\opt[k]}
{\opt[1:m]\in\times_{k=1}^m\optset[k]}
\text{ and }
\altoptset
=\cset[\bigg]{\smashoperator[r]{\sum_{\ell=1}^n}
\mu^{\altopt[1:n]}_\ell\altopt[\ell]}
{\altopt[1:n]\in\times_{\ell=1}^n\altoptset[\ell]}.
\end{equation}
For all $\opt[1:m]\in\times_{k=1}^m\optset[k]$ and $\altopt[1:n]\in\times_{\ell=1}^n\altoptset[\ell]$, we introduce the simplifying notation
\begin{equation*}
\aopt[{\opt[1:m]}]
\coloneqq\sum_{k=1}^m\lambda^{\opt[1:m]}_k\opt[k]
\text{ and }
\bopt[{\altopt[1:n]}]
\coloneqq\sum_{\ell=1}^n\mu^{\altopt[1:n]}_{\ell}\altopt[\ell],
\end{equation*}
so $\optset=\cset{\aopt[{\opt[1:m]}]}{\opt[1:m]\in \times_{k=1}^m\optset[k]}$ and $\altoptset=\cset{\bopt[{\altopt[1:n]}]}{\altopt[1:n]\in\times_{\ell=1}^n\altoptset[\ell]}$, and therefore
\begin{align*}
\altoptsettoo
&=\cset{\lambda_{\opt,\altopt}\opt+\mu_{\opt,\altopt}\altopt}{\opt\in\optset,\altopt\in\altoptset}\\
&=\cset[\bigg]{\lambda_{\aopt[{\opt[1:m]}],\bopt[{\altopt[1:n]}]}\aopt[{\opt[1:m]}]
+\mu_{\aopt[{\opt[1:m]}],\bopt[{\altopt[1:n]}]}\bopt[{\altopt[1:n]}]}
{\opt[1:m]\in\times_{k=1}^m\optset[k],\altopt[1:n]\in\times_{\ell=1}^n\altoptset[\ell]}.
\end{align*}
If we now introduce the notations
\begin{equation*}
\altoptsettoo[i]
\coloneqq
\begin{cases}
\optset[i] 
&\text{ if }1\leq i\leq m\\
\altoptset[i-m] 
&\text{ if }m+1\leq i\leq m+n
\end{cases}
\end{equation*}
and for any $\altopttoo[1:m+n]\in\times_{i=1}^{m+n}\altoptsettoo[i]$,
\begin{equation*}
\kappa_i^{\altopttoo[1:m+n]}
\coloneqq
\begin{cases}
\lambda_{\aopt[{\altopttoo[1:m]}],\bopt[{\altopttoo[m+1:m+n]}]}\lambda_i^{\altopttoo[1:m]} &\text{ if }1\leq i\leq m\\
\mu_{a_{\altopttoo[1:m]},b_{\altopttoo[m+1:m+n]}}\mu_{i-m}^{\altopttoo[m+1:m+n]} &\text{ if }m+1\leq i\leq m+n,
\end{cases}
\end{equation*}
where we used $\altopttoo[m+1:m+n]$ to denote the tuple $(\altopttoo[m+1],\dots,\altopttoo[m+n])$, then we find that
\begin{align*}
C
&=\cset[\bigg]{\sum_{i=1}^{m+n}\kappa_i^{\altopttoo[1:m+n]}\altopttoo[i]}
{\altopttoo[1:m+n]\in\times_{i=1}^{m+n}\altoptsettoo[i]}.
\end{align*}
Furthermore, since $(\lambda_{\aopt[{\altopttoo[1:m]}],\bopt[{\altopttoo[m+1:m+n]}]},\mu_{\aopt[{\altopttoo[1:m]}],\bopt[{\altopttoo[m+1:m+n]}]})>0$, $\lambda_{1:m}^{\altopttoo[1:m]}>0$ and $\mu_{1:n}^{\altopttoo[m+1:m+n]}>0$, it follows that also
\begin{equation*}
\kappa_{1:m+1}^{\altopttoo[1:m+n]}\coloneqq(\kappa_{1}^{\altopttoo[1:m+n]},\dots,\kappa_{m+1}^{\altopttoo[1:m+n]})>0.
\end{equation*}
Hence, we find that, indeed, $C\in\setposi(K)$.
\qed
\end{proof}

\begin{proposition}\label{prop:adding:supersets}
Consider any set of desirable gamble sets $\rejectset\in\rejectsets$.
Then\/ $\SU\group{\rejectset}$ satisfies Axiom~\ref{ax:rejects:mono}.
Moreover, if $\rejectset$ satisfies Axioms~\ref{ax:rejects:nonempty}, \ref{ax:rejects:pos} and\/~\ref{ax:rejects:cone} and does not contain~$\set{0}$, then so does\/ $\SU\group{\rejectset}$.
\end{proposition}

\begin{proof}
For the first statement, consider any $\optset[1]\in\SU\group{\rejectset}$ and any $\optset[2]\in\optsets$ such that $\optset[1]\subseteq\optset[2]$.
Then there is some $\altoptset[1]\in\rejectset$ such that $\altoptset[1]\subseteq\optset[1]$, and therefore also $\altoptset[1]\subseteq\optset[2]$, whence indeed $\optset[2]\in\SU\group{\rejectset}$.

For the second statement, assume that $\rejectset$ satisfies Axioms~\ref{ax:rejects:nonempty}, \ref{ax:rejects:pos} and~\ref{ax:rejects:cone} and does not contain $\set{0}$.


To prove that $\SU\group{\rejectset}$ satisfies Axiom~\ref{ax:rejects:pos}, simply observe that the operator $\SU$ never removes gamble sets from a set of desirable gamble sets, so the gamble sets $\set{\opt}$, $\opt\in\posopts$, that belong to $\rejectset$ by Axiom~\ref{ax:rejects:pos}, will also belong to the larger $\SU\group{\rejectset}$.

To prove that $\SU\group{\rejectset}$ satisfies Axiom~\ref{ax:rejects:cone}, consider any $\optset[1],\optset[2]\in\SU\group{\rejectset}$, meaning that there are $\altoptset[1],\altoptset[2]\in\rejectset$ such that $\altoptset[1]\subseteq\optset[1]$ and $\altoptset[2]\subseteq\optset[2]$.
For all $\opt\in\optset[1]$ and $\altopt\in\optset[2]$, choose some $(\lambda_{\opt,\altopt},\mu_{\opt,\altopt})>0$, and let 
\begin{equation*}
\optset
\coloneqq\cset{\lambda_{\opt,\altopt}\opt+\mu_{\opt,\altopt}\altopt}{\opt\in\optset[1],\altopt\in\optset[2]}.
\end{equation*}
We then need to prove that $\optset\in\SU\group{\rejectset}$. 
Since $\rejectset$ satisfies Axiom~\ref{ax:rejects:cone}, we infer from $\altoptset[1],\altoptset[2]\in\rejectset$ that
\begin{equation*}
\altoptset
\coloneqq\cset{\lambda_{\opt,\altopt}\opt+\mu_{\opt,\altopt}\altopt}
{\opt\in\altoptset[1],\altopt\in\altoptset[2]}\in\rejectset.
\end{equation*}
Since $\altoptset\subseteq\optset$, this implies that, indeed, $\optset\in\SU\group{\rejectset}$.

Finally, to prove that $\set{0}\notin\SU\group{\rejectset}$ and that $\SU\group{\rejectset}$ satisfies Axiom~\ref{ax:rejects:nonempty}, assume \emph{ex absurdo} that $\set{0}\in\SU\group{\rejectset}$ or $\emptyset\in\SU\group{\rejectset}$. Then $\set{0}\in\rejectset$ or $\emptyset\in\rejectset$. In either case, we obtain a contradiction with the assumption that $\rejectset$ satisfies Axiom~\ref{ax:rejects:nonempty} and does not contain $\set{0}$.
\qed
\end{proof}

\begin{proposition}\label{prop:ax:rejects:cone:equivalents}
Consider any coherent set of desirable gamble sets $\rejectset\in\cohrejectsets$, then $\setposi\group{\rejectset}=\rejectset$.
\end{proposition}

\begin{proof}
That $\rejectset\subseteq\setposi\group{\rejectset}$, is an immediate consequence of the definition of the $\setposi$ operator, and holds for any set of desirable gamble sets, coherent or not. 
Indeed, consider any $\optset\in\rejectset$, then it is not difficult to see that $\optset\in\setposi\group{\rejectset}$: choose $n\coloneqq1$, $\optset[1]\coloneqq\optset\in\rejectset^1$, and $\lambda^{\opt[1:1]}_{1:1}\coloneqq1$ for all $\opt[1:1]\in\times_{k=1}^1\optset[1]=\optset$ in the definition of the $\setposi$ operator.

For the converse inclusion, that $\setposi\group{\rejectset}\subseteq\rejectset$, we use the coherence of $\rejectset$, and in particular the representation result of Theorem~\ref{theo:rejectsets:representation}, which allows us to write that $\rejectset=\bigcap\cset{\rejectset[\desirset]}{\desirset\in\cohdesirsets\text{ and }\rejectset\subseteq\rejectset[\desirset]}$.

So, if we fix any $\desirset\in\cohdesirsets$ such that $\rejectset\subseteq\rejectset[\desirset]$, then it clearly suffices to prove that also $\setposi\group{\rejectset}\subseteq\rejectset[\desirset]$.
Consider, therefore, any $\optset\in\setposi(\rejectset)$, meaning that there are $n\in\naturals$, $(\optset[1],\dots,\optset[n])\in\rejectset^n$ and, for all $\opt[1:n]\in\times_{k=1}^n\optset[k]$, some choice of $\lambda_{1:n}^{\opt[1:n]}>0$ such that
\begin{equation*}
\optset=\cset[\bigg]{\sum_{k=1}^n\lambda_k^{\opt[1:n]}\opt[k]}{\opt[1:n]\in\times_{k=1}^n\optset[k]}.
\end{equation*} 
For any $k\in\set{1,\dots,n}$, since $\optset[k]\in\rejectset\subseteq\rejectset[\desirset]$, we know that $\optset[k]\cap\desirset\neq\emptyset$, so we can fix some $\altopt[k]\in\optset[k]\cap\desirset$. 
Then, on the one hand, we see that $\sum_{k=1}^n\lambda_k^{\altopt[1:n]}\altopt[k]\in\optset$.
On the other hand, since $\lambda_{1:n}^{\altopt[1:n]}>0$, we infer from Axiom~\ref{ax:desirs:cone} [by applying it multiple times] that also $\sum_{k=1}^n\lambda_k^{\altopt[1:n]}\altopt[k]\in\desirset$. Therefore, we find that $\optset\cap\desirset\neq\emptyset$, or equivalently, that $\optset\in\rejectset[\desirset]$.
Since $\optset\in\setposi(\rejectset)$ was chosen arbitrarily, it follows that, indeed, $\setposi\group{\rejectset}\subseteq\rejectset[\desirset]$.
\qed
\end{proof}

\begin{theorem}\label{theo:rejectsets:consistency}
An assessment $\assessment\subseteq\optsets$ is consistent if and only if\/ $\emptyset\notin\assessment$ and\/ $\set{0}\notin\setposi\group{\singposopts\cup\assessment}$.
\end{theorem}

\begin{proof}
For notational simplicity, we will denote the set of desirable gamble sets $\setposi\group{\singposopts\cup\assessment}$ by $\rejectset[o]$.

First, assume that $\emptyset\notin\assessment$ and $\set{0}\notin\rejectset[o]$.
Observe that $\rejectset[o]$ satisfies Axiom~\ref{ax:rejects:pos} by construction and Axiom~\ref{ax:rejects:cone} by Proposition~\ref{prop:applying:posi}. Furthermore, $\emptyset\notin\assessment$ implies that $\emptyset\notin\rejectset[o]$, and therefore, that $\rejectset[o]$ satisfies Axiom~\ref{ax:rejects:nonempty}. 
Since $\set{0}\notin\rejectset[o]$ by assumption, it therefore follows from Proposition~\ref{prop:adding:supersets} that $\SU\group{\rejectset[o]}$ satisfies Axioms~\ref{ax:rejects:nonempty}, \ref{ax:rejects:pos}, \ref{ax:rejects:cone} and~\ref{ax:rejects:mono}, and that $\set{0}\notin\SU\group{\rejectset[o]}$, so we gather from Proposition~\ref{prop:removal:of:nonpositives} that $\rejectset[1]\coloneqq\RN\group{\SU\group{\rejectset[o]}}$ satisfies Axioms~\ref{ax:rejects:nonempty}—-\ref{ax:rejects:mono}.
Since $\rejectset[1]$ includes $\assessment$ [none of the operators $\setposi$, $\SU$ and $\RN$ remove gamble sets from their arguments, they only add new gamble sets], this implies that $\rejectset[1]\in\cohrejectsets(\assessment)$, and therefore, that $\assessment$ is consistent.

Next, assume that $\assessment$ is consistent, which means that $\cohrejectsets(\assessment)\neq\emptyset$. Consider any $\rejectset\in \cohrejectsets(\assessment)$, which means that $\rejectset$ is coherent and $\assessment\subseteq\rejectset$. Then $\singposopts\cup\assessment\subseteq\rejectset$ [use Axiom~\ref{ax:rejects:pos}] and therefore also $\rejectset[o]=\setposi\group{\singposopts\cup\assessment}\subseteq\setposi\group{\rejectset}=\rejectset$ [for the inclusion, use the definition of the $\setposi$ operator, and for the equality, use Proposition~\ref{prop:ax:rejects:cone:equivalents}]. Now assume \emph{ex absurdo} that $\set{0}\in\rejectset[o]$. Then also $\set{0}\in\rejectset$ and therefore, Axiom~\ref{ax:rejects:removezero} implies that $\emptyset\in\rejectset$, which is impossible because of Axiom~\ref{ax:rejects:nonempty}. Hence, we find that $\set{0}\notin\rejectset[o]$. Finally, since $\rejectset$ is coherent, Axiom~\ref{ax:rejects:nonempty} implies that $\emptyset\notin\rejectset$, which, since $\assessment\subseteq\rejectset$, implies that $\emptyset\notin\assessment$.
\qed
\end{proof}

\begin{proposition}\label{prop:ax:rejects:SU:equivalents}
Consider any coherent set of desirable gamble sets $\rejectset\in\cohrejectsets$, then $\SU\group{\rejectset}=\rejectset$.
\end{proposition}

\begin{proof}
That $\rejectset\subseteq\SU\group{\rejectset}$ is an immediate consequence of the definition of the $\SU$ operator. 
The converse inclusion follows from the fact that $\rejectset$ is coherent and therefore satisfies Axiom~\ref{ax:rejects:mono}.
\qed
\end{proof}

\begin{theorem}\label{theo:rejectsets:natex}
For a consistent assessment $\assessment\subseteq\optsets$, $\natexrejectset(\assessment)=\RS\group{\setposi\group{\singposopts\cup\assessment}}$.
\end{theorem}

\begin{proof}
Assume that $\assessment$ is consistent. Theorem~\ref{theo:rejectsets:consistency} then tells us that $\emptyset\notin\assessment$ and $\set{0}\notin\rejectset[o]\coloneqq\setposi\group{\singposopts\cup\assessment}$.
We have to prove that $\natexrejectset(\assessment)=\RS\group{\rejectset[o]}$, or equivalently, due to Lemma~\ref{lem:combineoperators}, that $\natexrejectset(\assessment)=\rejectset[1]\coloneqq\RN\group{\SU\group{\rejectset[o]}}$.

We already know from the proof of Theorem~\ref{theo:rejectsets:consistency} that $\emptyset\notin\assessment$ and $\set{0}\notin\rejectset[o]$ implies that $\rejectset[1]\in\cohrejectsets(\assessment)$, and therefore also $\natexrejectset(\assessment)\subseteq\rejectset[1]$.
To prove the converse inclusion, consider any $\rejectset\in\cohrejectsets(\assessment)$.
Then as shown in the proof of Theorem~\ref{theo:rejectsets:consistency}, $\rejectset[o]\subseteq\rejectset$.
Hence also $\SU\group{\rejectset[o]}\subseteq\SU\group{\rejectset}=\rejectset$ [for the inclusion, use the definition of the $\SU$ operator, and for the equality, use Proposition~\ref{prop:ax:rejects:SU:equivalents}], and therefore also $\rejectset[1]=\RN\group{\SU\group{\rejectset[o]}}\subseteq\RN\group{\rejectset}=\rejectset$ [for the inclusion, use the definition of the $\RN$ operator, and for the equality, use Proposition~\ref{prop:ax:rejects:RN:equivalents}].
So $\rejectset[1]\subseteq\rejectset$. Since this is true for every $\rejectset\in\cohrejectsets(\assessment)$, and since the consistency of $\assessment$ implies that $\cohrejectsets(\assessment)\neq\emptyset$, we conclude that $\rejectset[1]\subseteq\natexrejectset(\assessment)$.
\qed
\end{proof}

\begin{proof}[Theorem~\ref{theo:rejectsets:consistency:and:natex}]
Immediately from Theorems~\ref{theo:rejectsets:consistency} and~\ref{theo:rejectsets:natex}.
\qed
\end{proof}

\end{document}